\documentclass{article}

    \PassOptionsToPackage{numbers, compress}{natbib}

\usepackage[final]{neurips_2019}

\usepackage[utf8]{inputenc} \usepackage[T1]{fontenc}    \usepackage{hyperref}       \usepackage{url}            \usepackage{booktabs}       \usepackage{amsfonts}       \usepackage{nicefrac}       \usepackage{microtype}      \usepackage{subcaption}
\usepackage{graphicx}
\usepackage{algorithm}
\usepackage{amsmath}
\usepackage{algorithmic}
\usepackage{xcolor}
\usepackage{wrapfig}
\usepackage{diagbox}

\newsavebox{\bigimage}

\newcommand{\bx}{\mathbf{x}}

\newcommand{\bA}{\mathbf{A}}
\newcommand{\bI}{\mathbf{I}}
\newcommand{\bc}{\mathbf{c}}

\newcommand{\E}{\mathop{\mathbb{E}}}
\newcommand{\V}{\mathop{\mathbb{V}}}
\newcommand{\btheta}{\boldsymbol{\theta}}
\newcommand{\bphi}{\boldsymbol{\phi}} \newif\ifcomments

\commentstrue

\ifcomments
  \newcommand{\colornote}[3]{{\color{#1}\bf{#2: #3}\normalfont}}
\else
  \newcommand{\colornote}[3]{}
\fi

\usepackage{amsthm}

\newtheorem{proposition}{Proposition}
\newtheorem{lemma}{Lemma}

\selectfont

\title{Lookahead Optimizer: $k$ steps forward, 1 step back}

\author{  Michael R. Zhang, 
  James Lucas,  
  Geoffrey Hinton, 
  Jimmy Ba\\
  Department of Computer Science, University of Toronto, Vector Institute \\
  \texttt{\{michael, jlucas, hinton,jba\}@cs.toronto.edu}
}

\begin{document}

\maketitle
\begin{abstract}
The vast majority of successful deep neural networks are trained using variants of stochastic gradient descent (SGD) algorithms. Recent attempts to improve SGD can be broadly categorized into two approaches: (1) adaptive learning rate schemes, such as AdaGrad and Adam, and (2) accelerated schemes, such as heavy-ball and Nesterov momentum. In this paper, we propose a new optimization algorithm, Lookahead, that is orthogonal to these previous approaches and iteratively updates two sets of weights.
Intuitively, the algorithm chooses a search direction by \emph{looking ahead} at the sequence of ``fast weights" generated by another optimizer.
We show that Lookahead improves the learning stability and lowers the variance of its inner optimizer with negligible computation and memory cost. We empirically demonstrate Lookahead can significantly improve the performance of SGD and Adam, even with their default hyperparameter settings on ImageNet, CIFAR-10/100, neural machine translation, and Penn Treebank.  

\end{abstract} \section{Introduction}
\label{sec:intro}
Despite their simplicity, SGD-like algorithms remain competitive for neural network training against advanced second-order optimization methods.  Large-scale distributed optimization algorithms \citep{goyal2017accurate, you2018imagenet} have shown impressive performance in combination with improved learning rate scheduling schemes \citep{vaswani2017attention,radford2018improving}, yet variants of SGD remain the core algorithm in the distributed systems. The recent improvements to SGD  can be broadly categorized into two approaches:  (1) adaptive learning rate schemes, such as AdaGrad~\citep{duchi2011adaptive} and Adam \citep{kingma2014adam}, and (2) accelerated schemes, such as Polyak heavy-ball~\citep{polyak1964some} and Nesterov momentum~\citep{nesterov1983method}.  Both approaches make use of the accumulated past gradient information to achieve faster convergence.  However, to obtain their improved performance in neural networks often requires costly hyperparameter tuning  \citep{montavon2012neural}. 

In this work, we present Lookahead, a new optimization method, that is orthogonal to these previous approaches.  Lookahead first updates the ``fast weights'' \citep{hinton1987using} $k$ times using any standard optimizer in its inner loop before updating the ``slow weights'' once in the direction of the final fast weights. We show that this update reduces the variance. We find that Lookahead is less sensitive to suboptimal hyperparameters and therefore lessens the need for extensive hyperparameter tuning. By using Lookahead with inner optimizers such as SGD or Adam, we achieve faster convergence across different deep learning tasks with minimal computational overhead.

Empirically, we evaluate Lookahead by  
training classifiers on the CIFAR \citep{krizhevsky2009learning} and ImageNet datasets \citep{deng2009imagenet}, observing faster convergence on the ResNet-50 and ResNet-152 architectures \citep{he2016deep}. We also trained LSTM language models on the Penn Treebank dataset \citep{marcus1993building} and Transformer-based \citep{vaswani2017attention} neural machine translation models on the WMT 2014 English-to-German dataset. For all tasks, using Lookahead leads to improved convergence over the inner optimizer and often improved generalization performance while being robust to hyperparameter changes. Our experiments demonstrate that Lookahead is robust to changes in the inner loop optimizer, the number of fast weight updates, and the slow weights learning rate.

 \section{Method}
\label{sec:method}

\begin{figure}[t]
\centering
\vspace{-0.5cm}
\begin{minipage}{0.45\textwidth}
    \hspace{-0.3in}
    \includegraphics[width=1.\textwidth]{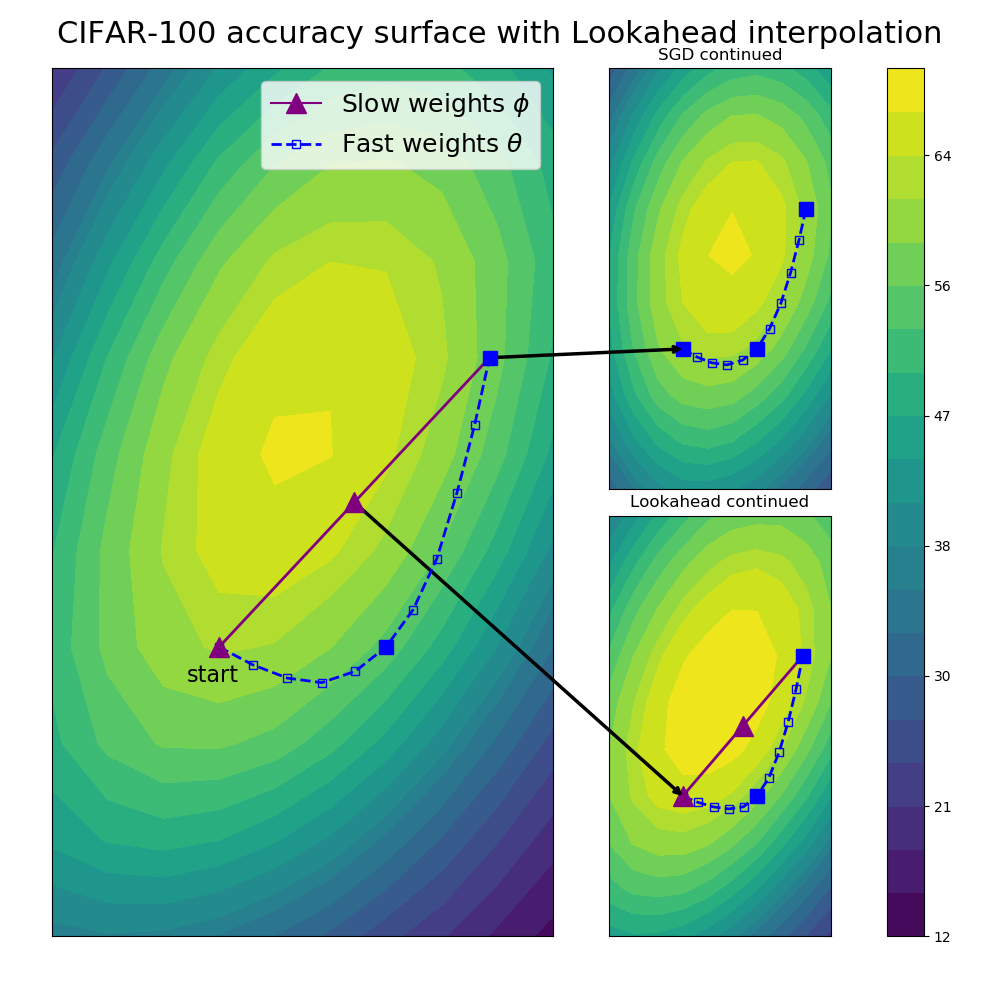}
        \vspace{-0.4cm}
\end{minipage}
\hfill
\hspace{-0.3in}
\begin{minipage}{0.55\textwidth}
\begin{algorithm}[H]
  \caption{Lookahead Optimizer:}
  \label{alg:lookahead}
\begin{algorithmic}
   \REQUIRE Initial parameters $\phi_0$, objective function $L$ 
   \REQUIRE Synchronization period $k$, slow weights step size $\alpha$, optimizer $A$
   \FOR{$t=1, 2, \dots$}
     \STATE Synchronize parameters $\theta_{t,0} \gets \phi_{t-1}$
     \FOR{$i=1, 2, \dots, k$}
        \STATE sample minibatch of data $d \sim \mathcal{D}$
        \STATE $\theta_{t,i} \gets \theta_{t,i-1} + A(L, \theta_{t,i-1}, d)$
     \ENDFOR
     \STATE Perform outer update $\phi_t \gets \phi_{t-1} + \alpha (\theta_{t,k} - \phi_{t-1})$
   \ENDFOR 
   \STATE \textbf{return} parameters $\phi$
\end{algorithmic}
\end{algorithm}
\end{minipage}
\vspace{0.1in}

\caption{({{Left}}) Visualizing Lookahead ($k=10$) through a ResNet-32 test accuracy surface at epoch 100 on CIFAR-100. We project the weights onto a plane defined by the first, middle, and last fast (inner-loop) weights. The fast weights are along the blue dashed path. All points that lie on the plane are represented as solid, including the entire Lookahead slow weights path (in purple). Lookahead (middle, bottom right) quickly progresses closer to the minima than SGD (middle, top right) is able to. ({{Right}}) Pseudocode for Lookahead. \label{fig:cifar_100_loss_viz}}
\end{figure}

In this section, we describe the Lookahead algorithm and discuss its properties. Lookahead maintains a set of slow weights $\phi$ and fast weights $\theta$, which get synced with the fast weights every $k$ updates. The fast weights are updated through applying $A$, any standard optimization algorithm, to batches of training examples sampled from the dataset $\mathcal{D}$. After $k$ inner optimizer updates using $A$, the slow weights are updated towards the fast weights by linearly interpolating in weight space, $\theta - \phi$. We denote the slow weights learning rate as $\alpha$. After each slow weights update, the fast weights are reset to the current slow weights value. Psuedocode is provided in Algorithm $\ref{alg:lookahead}$.\footnote{Our open source implementation is available at \url{https://github.com/michaelrzhang/lookahead}.}

Standard optimization methods typically require carefully tuned learning rates to prevent oscillation and slow convergence. This is even more important in the stochastic setting \citep{martens2014new, wu2018understanding}. Lookahead, however, benefits from a larger learning rate in the inner loop. When oscillating in the high curvature directions, the fast weights updates make rapid progress along the low curvature directions. The slow weights help smooth out the oscillations through the parameter interpolation. The combination of fast weights and slow weights improves learning in high curvature directions, reduces variance, and enables Lookahead to converge rapidly in practice.  

Figure~\ref{fig:cifar_100_loss_viz} shows the trajectory of both the fast weights and slow weights during the optimization of a ResNet-32 model on CIFAR-100. While the fast weights explore around the minima, the slow weight update pushes Lookahead aggressively towards an area of improved test accuracy, a region which remains unexplored by SGD after 20 updates. 

\paragraph{Slow weights trajectory} We can characterize the trajectory of the slow weights as an exponential moving average (EMA) of the \emph{final} fast weights within each inner-loop, regardless of the inner optimizer. {After $k$ inner-loop steps we have: \begin{align}
    \phi_{t+1} &= \phi_{t} + \alpha (\theta_{t,k} - \phi_{t}) \\
     &= \alpha [\theta_{t,k} + (1-\alpha)\theta_{t-1,k} + \ldots + (1-\alpha)^{t-1} \theta_{0,k}] + (1-\alpha)^t \phi_0
\end{align}}
Intuitively, the slow weights heavily utilize recent proposals from the fast weight optimization but maintain some influence from previous fast weights. We show that this has the effect of reducing variance in Section \ref{sec: noisyquadratic}. While a Polyak-style average has further theoretical guarantees, our results match the claim that ``an exponentially-decayed moving average typically works much better in practice" \citep{martens2014new}.

\paragraph{Fast weights trajectory} Within each inner-loop, the trajectory of the fast weights depends on the choice of underlying optimizer. Given an optimization algorithm $A$ that takes in an objective function $L$ and the current mini-batch training examples $d$, we have the update rule for the fast weights:
\begin{equation}
\theta_{t, i+1} = \theta_{t, i} + A(L, \theta_{t, i-1},d).
\end{equation}
We have the choice of maintaining, interpolating, or resetting the internal state (e.g. momentum) of the inner optimizer. We evaluate this tradeoff on the CIFAR dataset (where every choice improves convergence) in Appendix \ref{app:interpolation} and maintain internal state for the other experiments.

\paragraph{Computational complexity} Lookahead has a constant computational overhead due to parameter copying and basic arithmetic operations that is amortized across the $k$ inner loop updates. The number of operations is $\mathcal{O}(\frac{k+1}{k})$ times that of the inner optimizer. Lookahead maintains a single additional copy of the number of learnable parameters in the model.

\begin{figure}
    \centering
    \begin{minipage}{0.49 \linewidth}
    \includegraphics[width=\linewidth]{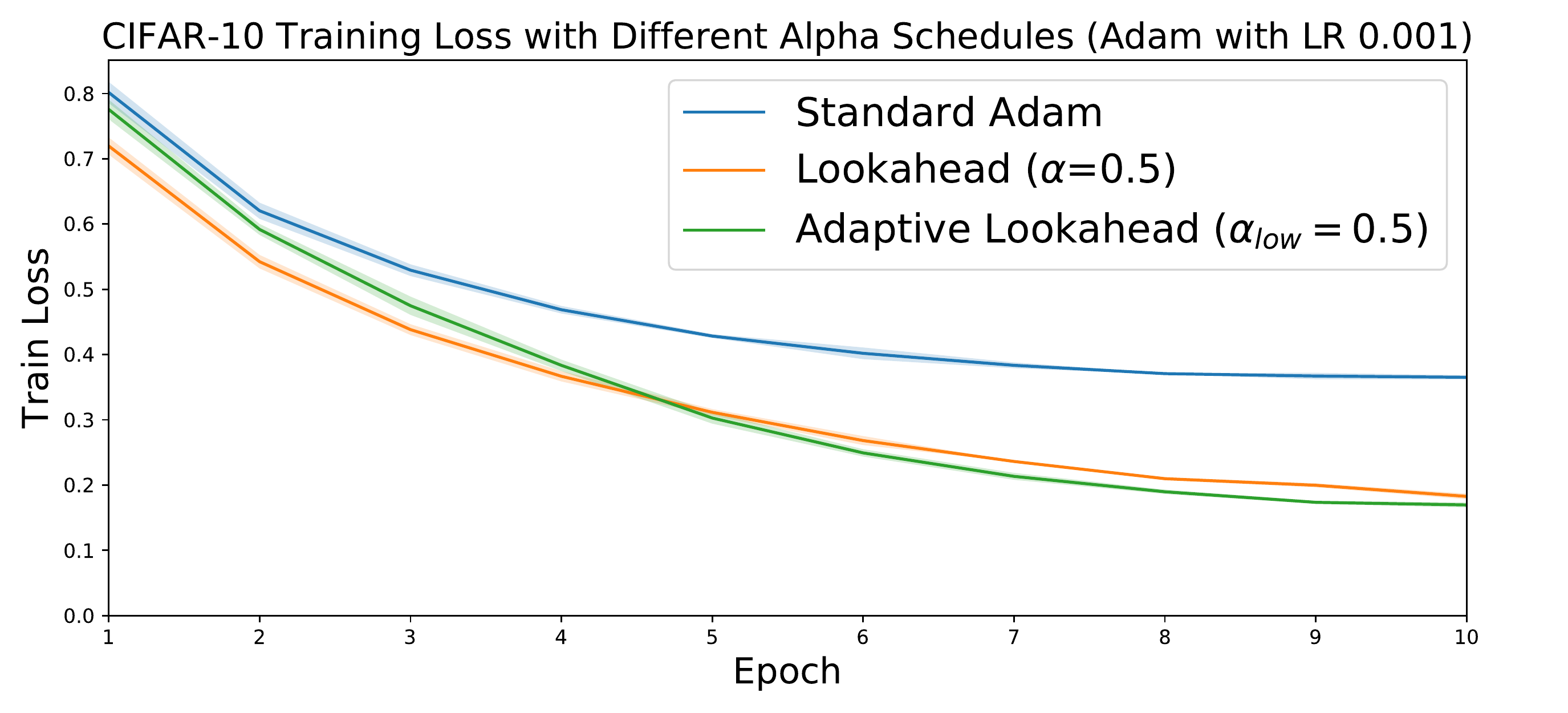}
    \end{minipage} \hfill
    \begin{minipage}{0.49 \linewidth}
    \includegraphics[width=\linewidth]{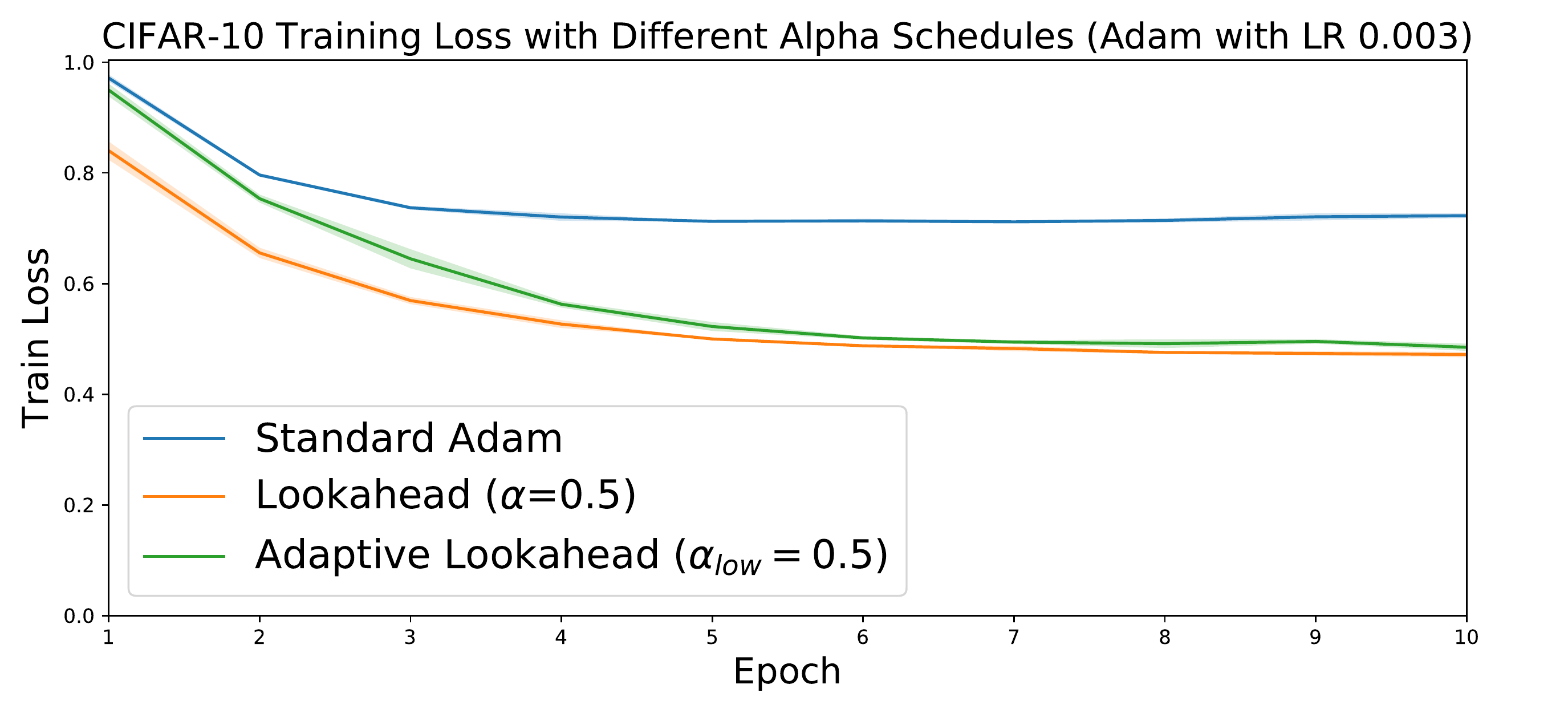}
    \end{minipage}
    \caption{CIFAR-10 training loss with fixed and adaptive $\alpha$. The adaptive $\alpha$ is clipped between $[\alpha_{low}, 1]$. (Left) Adam learning rate = 0.001.  (Right) Adam learning rate = 0.003.} 
    \label{fig:adaptivealpha}
\end{figure}

\subsection{Selecting the Slow Weights Step Size}

The step size in the direction ($\theta_{t,k} - \theta_{t,0}$) is controlled by $\alpha$. By taking a quadratic approximation of the loss, we present a principled way of selecting $\alpha$. 

\begin{proposition}[Optimal slow weights step size]
\label{prop:optimal-step-size}
For a quadratic loss function $L(x) = \frac{1}{2} x^T A x - b^T x$, 
the step size $\alpha^*$ that minimizes the loss for two points $\theta_{t,0}$ and $\theta_{t,k}$ is given by:
\[ \alpha^* = \arg \min_{\alpha} L(\theta_{t,0} + \alpha (\theta_{t,k} - \theta_{t,0})) = \frac{(\theta_{t,0} - \theta^*)^T A(\theta_{t,0} - \theta_{t, k})}{(\theta_{t,0}  - \theta_{t,k})^T A (\theta_{t,0} - \theta_{t,k})}  \]
where $\theta^* = A^{-1}b$ minimizes the loss.
\end{proposition}
Proof is in the appendix. Using quadratic approximations for the curvature, which is typical in second order optimization \citep{duchi2011adaptive, kingma2014adam, martens2015optimizing}, we can derive an estimate for the optimal $\alpha$ more generally. The full Hessian is typically intractable so we instead use aforementioned approximations, such as the diagonal approximation to the empirical Fisher used by the Adam optimizer \citep{kingma2014adam}. This approximation works well in our numerical experiments if we clip the magnitude of the step size. At each slow weight update, we compute:

\[ \hat{\alpha}^* = \text{clip}( \frac{(\theta_{t,0} - (\theta_{t,k}  - \hat{A}^{-1} \hat{\nabla} L(\theta_{t,k}) )^T \hat{A} (\theta_{t,0}  - \theta_{t,k})}{(\theta_{t,0}  - \theta_{t,k})^T \hat{A} (\theta_{t,0} - \theta_{t,k})},  \alpha_{\text{low}}, 1)\]

 where $\hat{A}$ is the empirical Fisher approximation and $\theta_{t,k}  - \hat{A}^{-1} \hat{\nabla} L(\theta_{t,k})$ approximates the optimum $\theta^*$. We prove Proposition \ref{prop:optimal-step-size} and elaborate on assumptions in the appendix \ref{app:optimal-slow}. Setting $\alpha_{\text{low}} > 0$ improves the stability of our algorithm. We evaluate the performance of this adaptive scheme versus a fixed scheme and standard Adam on a ResNet-18 trained on CIFAR-10 with two different learning rates and show the results in Figure \ref{fig:adaptivealpha}. Additional hyperparameter details are given in appendix \ref{app:experiments}. Both the fixed and adaptive Lookahead offer improved convergence.

In practice, a fixed choice of $\alpha$ offers similar convergence benefits and tends to generalize better. Fixing $\alpha$ avoids the need to maintain an estimate of the empirical Fisher, which incurs a memory and computational cost when the inner optimizer does not maintain such an estimate e.g. SGD. We thus use a fixed $\alpha$ for the rest of our deep learning experiments. 

 \section{Convergence Analysis}

\subsection{Noisy quadratic analysis}
\label{sec: noisyquadratic}

\begin{wrapfigure}{t}{0.5 \textwidth}
    \centering
    \vspace{-0.4cm}    \includegraphics[width=0.9\linewidth]{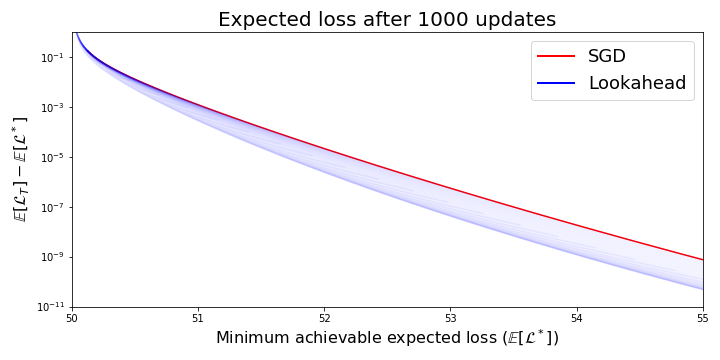}    \vspace{-0.1cm}    \caption{Comparing expected optimization progress between SGD and Lookahead($k=5$) on the noisy quadratic model.  Each vertical slice compares the convergence of optimizers with the same final loss values. For Lookahead, convergence rates for 100 evenly spaced $\alpha$ values in the range $(0,1]$ are overlaid.}
    \label{fig:noisy_quad_convergence}
    \vspace{-0.4cm}\end{wrapfigure}

We analyze Lookahead on a noisy quadratic model to better understand its convergence guarantees. While simple, this model is a proxy for neural network optimization and effectively optimizing it remains a challenging open problem \citep{schaul2013no, martens2015optimizing, wu2018understanding, zhang2019algorithmic}. In this section, we will show under equal learning rates that Lookahead will converge to a smaller steady-state risk than SGD. We will then show through simulation of the expected dynamics that Lookahead is able to converge to this steady-state risk more quickly than SGD for a range of hyperparameter settings.

\paragraph{Model definition} We use the same model as in \citet{schaul2013no} and \citet{wu2018understanding}. 
\begin{equation}
    \hat{\mathcal{L}}(\bx) = \frac{1}{2} (\bx - \bc)^T \bA (\bx - \bc),
\end{equation}
with $\bc \sim \mathcal{N}(\bx^*, \Sigma)$. We assume that both $\bA$ and $\Sigma$ are diagonal and that, without loss of generality, $\bx^* = \mathbf{0}$. While it is trivial to assume that $\bA$ is diagonal \footnote{Classical momentum's iterates are invariant to translations and rotations (see e.g. \citet{sutskever2013importance})  and Lookahead's linear interpolation is also invariant to such changes.} the co-diagonalizable noise assumption is non-trivial but is common --- see \citet{wu2018understanding} and \citet{zhang2019algorithmic} for further discussion. We use $a_i$ and $\sigma_i^2$ to denote the diagonal elements of $\bA$ and $\Sigma$ respectively. Taking the expectation over $\bc$, the expected loss of the iterates $\theta^{(t)}$ is,

\begin{equation}\label{eqn:noisy_exp_loss}
    \mathcal{L}(\theta^{(t)}) = \E[\hat{\mathcal{L}}(\theta^{(t)})]
    = \frac{1}{2} \E [\sum_i a_i ({\theta_i^{(t)}}^2 + \sigma_i^2)]
    = \frac{1}{2}\sum_i a_i (\E[\theta_i^{(t)}]^2 + \V[\theta_i^{(t)}] + \sigma_i^2).
\end{equation}

Analyzing the expected dynamics of the SGD iterates and the slow weights gives the following result.

\begin{proposition}[Lookahead steady-state risk]\label{prop:noisy_quad}
Let $0 < \gamma < 2/L$ be the learning rate of SGD and Lookahead where $L=\max_i a_i$. In the noisy quadratic model, the iterates of SGD and Lookahead with SGD as its inner optimizer converge to $0$ in expectation and the variances converge to the following fixed points:
\begin{align}
    V_{SGD}^* &= \frac{\gamma^2 \bA^2 \Sigma^2}{\bI - (\bI - \gamma \bA)^2} \\
    V_{LA}^* &= \frac{\alpha^2 (\bI - (\bI - \gamma\bA)^{2k})}{\alpha^2 (\bI - (\bI - \gamma\bA)^{2k}) + 2\alpha(1 - \alpha)(\bI - (\bI - \gamma\bA)^k)} V_{SGD}^*
        \end{align}
\end{proposition}

\paragraph{Remarks}For the Lookahead variance fixed point, the first product term is always smaller than 1 for $\alpha \in (0,1)$, and thus Lookahead has a variance fixed point that is strictly smaller than that of the SGD inner-loop optimizer for the same learning rate. Evidence of this phenomenon is present in deep neural networks trained on the CIFAR dataset, shown in Figure~\ref{fig:cifar-every-batch}.

In Proposition~\ref{prop:noisy_quad}, we use the same learning rate for both SGD and Lookahead. To fairly evaluate the convergence of the two methods, we compare the convergence rates under hyperparameter settings that achieve the same steady-state risk. In Figure~\ref{fig:noisy_quad_convergence} we show the expected loss after 1000 updates (computed analytically) for both Lookahead and SGD. This shows that there exists (fixed) settings of the Lookahead hyperparameters that arrive at the same steady state risk as SGD but do so more quickly. Moreover, Lookahead outperforms SGD across the broad spectrum of $\alpha$ values we simulated. Details, further simulation results, and additional discussion are presented in Appendix~\ref{app:quadratic}.

 \subsection{Deterministic quadratic convergence}

In the previous section we showed that on the noisy quadratic model, Lookahead is able to improve convergence of the SGD optimizer under setting with equivalent convergent risk. Here we analyze the quadratic model without noise using gradient descent with momentum \citep{polyak1964some, goh2017why} and show that when the system is under-damped, Lookahead is able to improve on the convergence rate.

As before, we restrict our attention to diagonal quadratic functions (which in this case is entirely without loss of generality). Given an initial point $\btheta_0$, we wish to find the rate of contraction, that is, the smallest $\rho$ satisfying $||\btheta_t - \btheta^*|| \leq \rho^t ||\btheta_{0} - \btheta^*||$. We follow the approach of \citep{o2015adaptive} and model the optimization of this function as a linear dynamical system allowing us to compute the rate exactly. Details are in Appendix \ref{app:quadratic}.

\begin{wrapfigure}{r}{0.6 \textwidth}
    \centering
    \vspace{-0.2cm}
    \includegraphics[width=0.9 \linewidth]{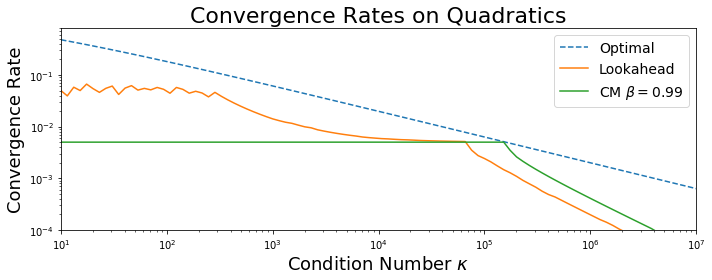}
    \vspace{-0.2cm}
    \caption{Quadratic convergence rates ($1-\rho$) of classical momentum versus Lookahead wrapping classical momentum. For Lookahead, we fix $k=20$ lookahead steps and $\alpha=0.5$ for the slow weights step size. Lookahead is able to significantly improve on the convergence rate in the under-damped regime where oscillations are observed.}
    \label{fig:quad_convergence}
\end{wrapfigure}

As in \citet{lucas2018aggregated}, to better understand the sensitivity of Lookahead to misspecified conditioning we fix the momentum coefficient of classical momentum and explore the convergence rate over varying condition number under the optimal learning rate.
As expected, Lookahead has slightly worse convergence in the over-damped regime where momentum is set too low and CM is slowly, monotonically converging to the optimum. However, when the system is under-damped (and oscillations occur) Lookahead is able to significantly improve the convergence rate by skipping to a better parameter setting during oscillation.

  \section{Related work}
\label{sec:related}
Our work is inspired by recent advances in understanding the loss surface of deep neural networks. While the idea of following the trajectory of weights dates back to \citet{ruppert1988efficient, polyak1992acceleration}, averaging weights in neural networks has not been carefully studied until more recently. \citet{garipov2018loss} observe that the final weights of two independently trained neural networks can be connected by a curve with low loss. \citet{izmailov2018averaging} proposes Stochastic Weight Averaging (SWA), which averages the weights at different checkpoints obtained during training. Parameter averaging schemes are used to create ensembles in natural language processing tasks \citep{jean2014using, merity2017regularizing} and in training Generative Adversarial Networks \citep{yasin2018averaging}. In contrast to previous approaches, which generally focus on generating a set of parameters at the \emph{end} of training, Lookahead is an optimization algorithm which performs parameter averaging \emph{during} the training procedure to achieve faster convergence. We elaborate on differences with SWA and present additional experimental results in appendix \ref{app:swa-comparison}.

The Reptile algorithm, proposed by \citet{nichol2018reptile}, samples tasks in its outer loop and runs an optimization algorithm on each task within the inner loop. The initial weights are then updated in the direction of the new weights. While the functionality is similar, the application and setting are starkly different. Reptile samples different tasks and aims to find parameters which act as good initial values for new tasks sampled at test time. Lookahead does not sample new tasks for each outer loop and aims to take advantage of the geometry of loss surfaces to improve convergence. 

Katyusha \citep{allen2017katyusha}, an accelerated form of SVRG \citep{johnson2013accelerating}, also uses an outer and inner loop during optimization. Katyusha checkpoints parameters during optimization. Within each inner loop step, the parameters are pulled back towards the latest checkpoint. Lookahead computes the pullback only at the end of the inner loop and the gradient updates do not utilize the SVRG correction (though this would be possible). While Katyusha has theoretical guarantees in the convex optimization setting, the SVRG-based update does not work well for neural networks \citep{1812.04529}. 

Anderson acceleration \citep{anderson1965iterative} and other related extrapolation techniques \citep{brezinski2013extrapolation} have a similar flavor to Lookahead. These methods keep track of all iterates within an inner loop and then compute some linear combination which extrapolates the iterates towards their fixed point. This presents additional challenges first in the form of additional memory overhead as the number of inner-loop steps increases and also in finding the best linear combination. \citet{scieur2018nonlinear, scieur2018nonlinear2} propose a  method by which to find a good linear combination and apply this approach to deep learning problems and report both improved convergence and generalization. However, their method requires on the order of $k$ times more memory than Lookahead. Lookahead can be seen as a simple version of Anderson acceleration wherein only the first and last iterates are used.

 \section{Experiments}
\label{sec:experiments}

We completed a thorough evaluation of the Lookahead optimizer on a variety of deep learning tasks against well-calibrated baselines. We explored image classification on CIFAR-10/CIFAR-100 \citep{krizhevsky2009learning} and ImageNet \citep{deng2009imagenet}. We also trained LSTM language models on the Penn Treebank dataset \citep{marcus1993building} and Transformer-based \citep{vaswani2017attention} neural machine translation models on the WMT 2014 English-to-German dataset. For all of our experiments, every algorithm consumed the same amount of training data.

\subsection{CIFAR-10 and CIFAR-100} 
\label{sec:cifar}

The CIFAR-10 and CIFAR-100 datasets for classification consist of $32 \times 32$ color images, with 10 and 100 different classes, split into a training set with 50,000 images and a test set with 10,000 images. We ran all our CIFAR experiments with 3 seeds and trained for 200 epochs on a ResNet-18 \citep{he2016deep} with batches of 128 images and decay the learning rate by a factor of $5$ at the 60th, 120th, and 160th epochs.  Additional details are given in appendix \ref{app:experiments}.

We summarize our results in Figure \ref{tabel:cifar-optimizer-compare}.\footnote{We refer to SGD with heavy ball momentum \citep{polyak1964some} as SGD.} We also elaborate on how Lookahead contrasts with SWA and present results demonstrating lower validation error with Pre-ResNet-110 and Wide-ResNet-28-10 \cite{Zagoruyko2016WRN} on CIFAR-100 in appendix \ref{app:swa-comparison}.  Note that Lookahead achieves significantly faster convergence throughout training even though the learning rate schedule is optimized for the inner optimizer---future work can involve building a learning rate schedule for Lookahead. This improved convergence is important for better anytime performance in new datasets where hyperparameters and learning rate schedules are not well-calibrated.

\begin{figure}[t]
\centering
\begin{minipage}[t]{0.48 \linewidth}
    \centering
    \includegraphics[width=1. \textwidth]{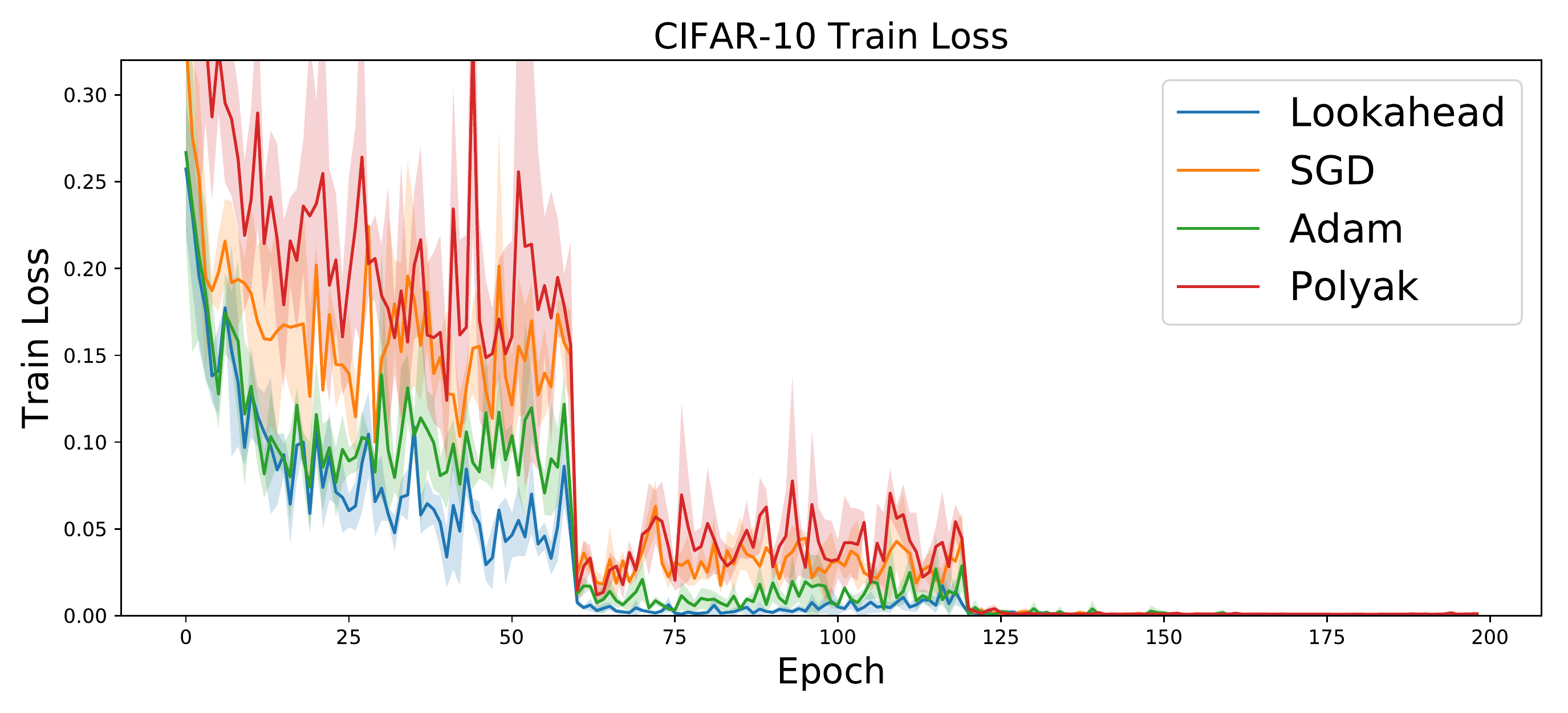}
\end{minipage}
\hfill
\begin{minipage}[t]{0.48 \linewidth}
\begin{table}[H]
\begin{center}
\vspace{-1.4in}
\begin{sc}
\begin{tabular}{l c c c c }
\toprule
Optimizer & CIFAR-10 & CIFAR-100  \\
\midrule
SGD & $95.23 \pm .19$ & $78.24 \pm .18$ \\ 
Polyak & $95.26 \pm .04$ & $77.99 \pm .42$ \\
Adam & $94.84 \pm .16$ & $76.88 \pm .39$ \\ 
Lookahead & $95.27 \pm .06$ & $78.34 \pm .05$ \\
\bottomrule
\end{tabular}
\end{sc}
\end{center}
\caption{CIFAR Final Validation Accuracy.}
\end{table}
\end{minipage}
\caption{Performance comparison of the different optimization algorithms. ({\bf{Left}}) Train Loss on CIFAR-100. ({\bf{Right}}) CIFAR ResNet-18 validation accuracies with various optimizers. We do a grid search over learning rate and weight decay on the other optimizers (details in appendix \ref{app:experiments}). Lookahead and Polyak are wrapped around SGD.}
\label{tabel:cifar-optimizer-compare}
\end{figure}

 \subsection{ImageNet} 

The 1000-way ImageNet task \citep{deng2009imagenet} is a classification task that contains roughly 1.28 million training images and 50,000 validation images. We use the official PyTorch implementation\footnote{Implementation available at \url{https://github.com/pytorch/examples/tree/master/imagenet}.} and the ResNet-50 and ResNet-152  \citep{he2016deep} architectures. Our baseline algorithm is SGD with an initial learning rate of 0.1 and momentum value of 0.9. We train for 90 epochs and decay our learning rate by a factor of 10 at the 30th and 60th epochs. For Lookahead, we set $k=5$ and slow weights step size $\alpha = 0.5$.

Motivated by the improved convergence we observed in our initial experiment, we tried a more aggressive learning rate decay schedule where we decay the learning rate by a factor of 10 at the 30th, 48th, and 58th epochs. Using such a schedule, we reach $75\%$ single crop top-1 accuracy on ImageNet in just 50 epochs and reach $75.5\%$ top-1 accuracy in 60 epochs. The results are shown in Figure \ref{fig:imagenet-visual}.

\begin{figure}[t]
    \centering
    \begin{minipage}{0.48 \linewidth}
    \includegraphics[width=\linewidth]{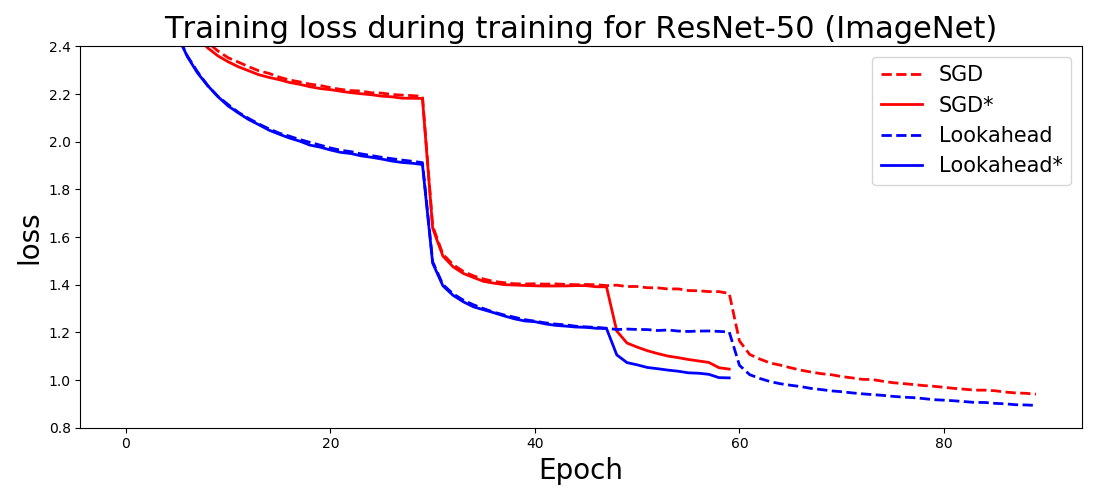}
    \end{minipage} \hfill
    \begin{minipage}{0.48 \linewidth}
\begin{table}[H]
\label{tabel:imagenet-compare}
\begin{center}
\begin{small}
\begin{sc}
\begin{tabular}{l c c }
\toprule
Optimizer & LA & SGD   \\
\midrule
Epoch 50 - Top 1 & 75.13 & 74.43  \\
Epoch 50 - Top 5 & 92.22 & 92.15  \\ 
Epoch 60 - Top 1 & 75.49 & 75.15  \\
Epoch 60 - Top 5 & 92.53 & 92.56  \\ 
\bottomrule
\end{tabular}
\end{sc}
\end{small}
\end{center}
\caption{Top-1 and Top-5 single crop validation accuracies on ImageNet.}
\vskip -0.1in
\end{table}
\end{minipage}
\caption{ImageNet training loss. The asterisk denotes the aggressive learning rate decay schedule, where LR is decayed at iteration 30, 48, and 58. We report validation accuracies for this schedule.}
\vskip -0.1in
\label{fig:imagenet-visual}
    \end{figure}

To test the scalability of our method, we ran Lookahead with the aggressive learning rate decay on ResNet-152. We reach $77\%$ single crop top-1 accuracy in 49 epochs (matching what is reported in \citet{he2016deep}) and $77.96\%$ top-1 accuracy in 60 epochs. Other approaches for improving convergence on ImageNet can require hundreds of GPUs, or tricks such as ramping up the learning rate and adaptive batch-sizes \citep{goyal2017accurate, jia2018highly}. The fastest convergence we are aware of uses an approximate second-order method to train a ResNet-50 to $75\%$ top-1 accuracy in 35 epochs with 1,024 GPUs \citep{1811.12019}. In contrast, Lookahead requires changing one single line of code and can easily scale to ResNet-152.  \subsection{Language modeling}
\label{sec:lm}

We trained LSTMs \citep{hochreiter1997long} for language modeling on the Penn Treebank dataset. We followed the model setup of \citet{merity2017regularizing} and made use of their publicly available code in our experiments. We did not include the fine-tuning stages. We searched over hyperparameters for both Adam and SGD (without momentum) to find the model which gave the best validation performance. We then performed an additional small grid search on each of these methods with Lookahead. Each model was trained for 750 epochs. We show training curves for each model in Figure~\ref{fig:lstm-train}.

\begin{figure}[t]
\begin{subfigure}{0.45 \textwidth}
\centering
    \vskip -0.1in
    \includegraphics[width=1. \textwidth]{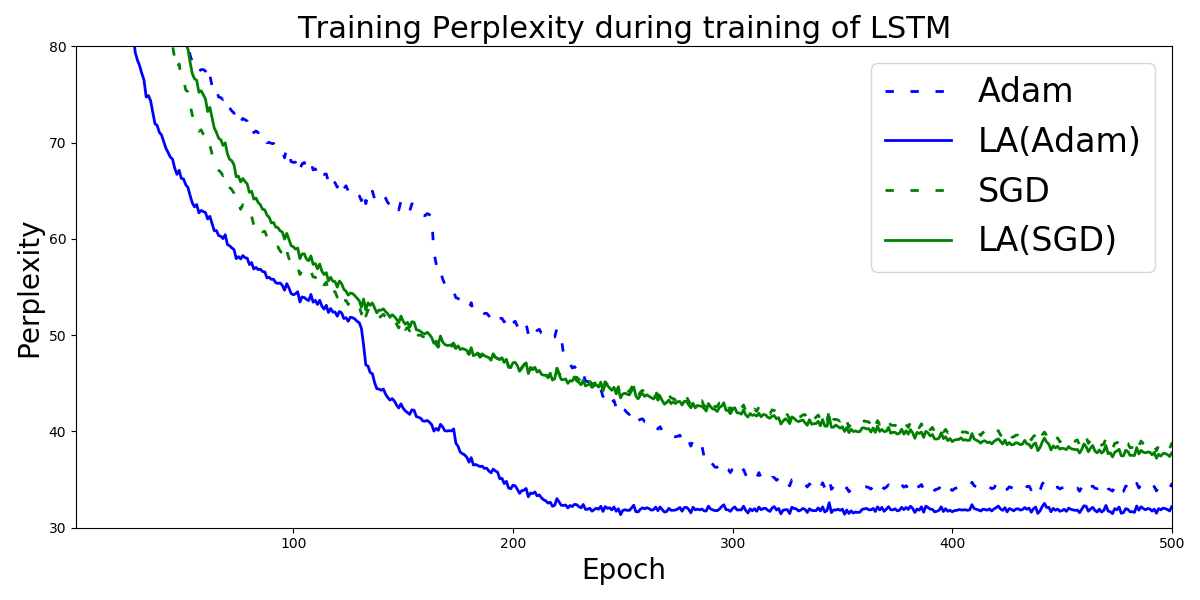}
    \caption{Training perplexity of LSTM models trained on the Penn Treebank dataset}
    \label{fig:lstm-train}
\end{subfigure}
\hfill
\hspace{0.05in}
\begin{subfigure}{0.45 \textwidth}
    \centering
    \includegraphics[width=1. \textwidth]{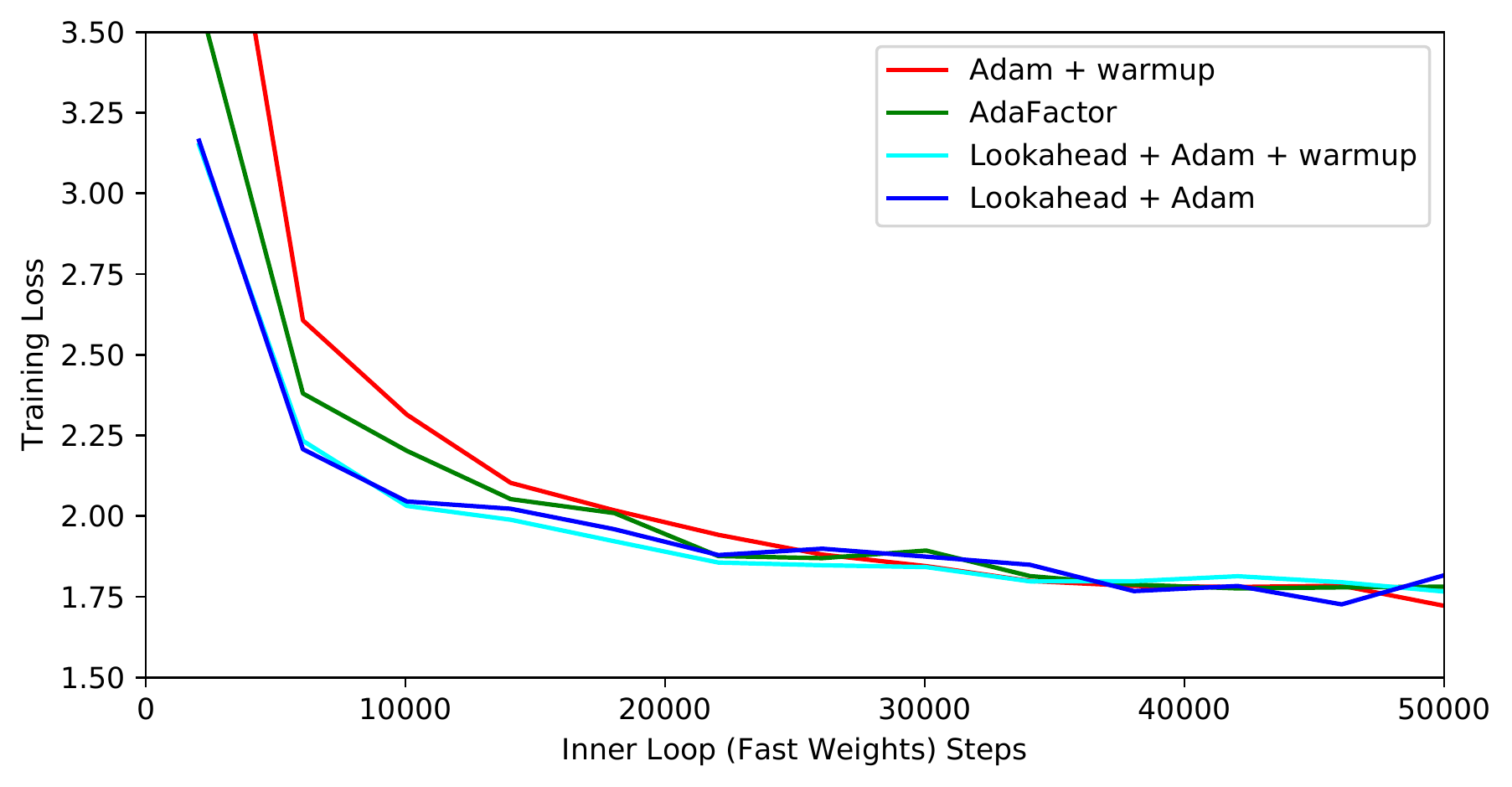}
    \caption{Training Loss on Transformer. Adam and AdaFactor both use a linear warmup scheme described in \citet{vaswani2017attention}.}
    \label{fig:nmt}
\end{subfigure}
\caption{Optimization performance on Penn Treebank and WMT-14 machine translation task.}
\end{figure}

\begin{figure}[t]
\begin{minipage}[t]{0.48 \linewidth}
\begin{table}[H]
\caption{LSTM training, validation, and test perplexity on the Penn Treebank dataset.}
\label{tab:lstm}
\vskip 0.1in
\begin{center}
\begin{small}
\begin{sc}
\begin{tabular}{l c c c }
\toprule
Optimizer & Train & Val. & Test  \\
\midrule
SGD & 43.62 & 66.0 & 63.90 \\
LA(SGD) & 35.02 & 65.10 & 63.04 \\ 
Adam & 33.54 & 61.64 & 59.33 \\
LA(Adam) & \textbf{31.92} & \textbf{60.28} & \textbf{57.72} \\
Polyak & - & 61.18 & 58.79 \\ 
\bottomrule
\end{tabular}
\end{sc}
\end{small}
\end{center}
\vskip -0.1in
\end{table}
\end{minipage}
\hfill
\begin{minipage}[t]{0.48 \linewidth}
\begin{table}[H]
\caption{Transformer Base Model trained for 50k steps on WMT English-to-German. ``Adam-'' denote Adam without learning rate warm-up.}
\label{tab:ntm}
\vskip 0.1in
\begin{center}
\begin{small}
\begin{sc}
\begin{tabular}{l c c }
\toprule
Optimizer & Newstest13 & Newstest14   \\
\midrule
Adam & 24.6 & 24.6 \\
LA(Adam) & 24.68 & 24.70 \\ 
LA(Adam-) & 24.3 & 24.4 \\
AdaFactor & 24.17 & 24.51 \\
\bottomrule
\end{tabular}
\end{sc}
\end{small}
\end{center}
\vskip -0.1in
\end{table}
\end{minipage}
\end{figure}

Using Lookahead with Adam we were able to achieve the fastest convergence and best training, validation, and test perplexity. The models trained with SGD took much longer to converge (around 700 epochs) and were unable to match the final performance of Adam. Using Polyak weight averaging \citep{polyak1992acceleration} with SGD, as suggested by \citet{merity2017regularizing} and referred to as ASGD, we were able to improve on the performance of Adam but were unable to match the performance of Lookahead.  Full results are given in Table~\ref{tab:lstm} and additional details are in appendix \ref{app:experiments}.
 \subsection{Neural machine translation}
\label{sec:nmt}

We trained Transformer based models \citep{vaswani2017attention} on the WMT2014 English-to-German translation task on a single Tensor Processing Unit (TPU) node. We took the base model from \citet{vaswani2017attention} and trained it using the proposed warmup-then-decay learning rate scheduling scheme and, additionally, the same scheme wrapped with Lookahead. We found Lookahead speedups the early stage of the training over Adam and the later proposed AdaFactor \citep{adafactor} optimizer. All the methods converge to similar training loss and BLEU score at the end, see Figure~\ref{fig:nmt} and Table~\ref{tab:ntm}. 

Our NMT experiments further confirms Lookahead improves the robustness of the inner loop optimizer. We found Lookahead enables a wider range of learning rate \{0.02, 0.04, 0.06\} choices for the Transformer model that all converge to similar final losses. Full details are given in Appendix~\ref{app:nmt}.
 \subsection{Empirical analysis}
\label{sec:empirical_analysis}

\begin{figure*}
        \centering
        \begin{subfigure}[b]{0.475\textwidth}
            \centering
            \includegraphics[width=\textwidth]{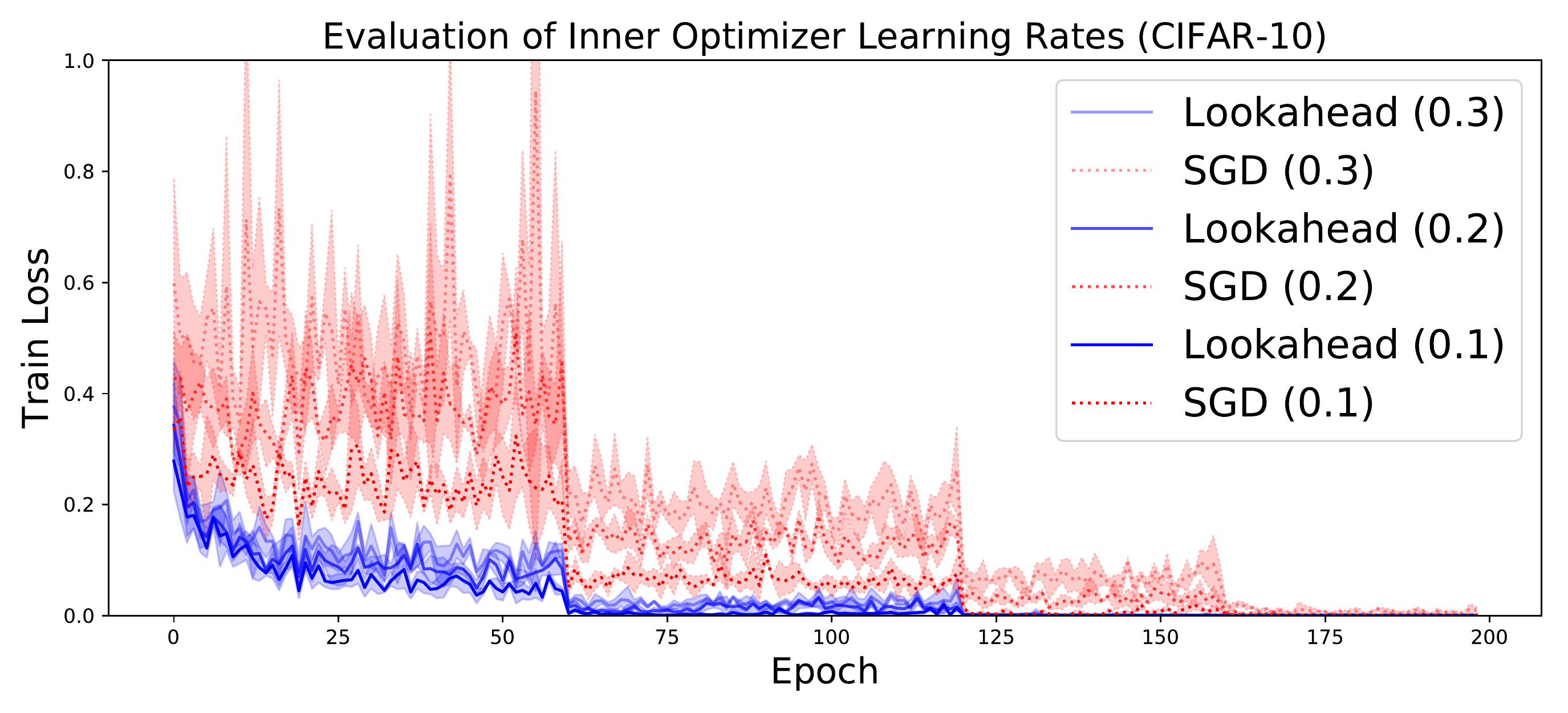}
            \caption[Network2]            {{\small CIFAR-10 Train Loss: Different LR}}    
                    \end{subfigure}
        \hfill
        \begin{subfigure}[b]{0.475\textwidth}  
            \centering 
            \includegraphics[width=\textwidth]{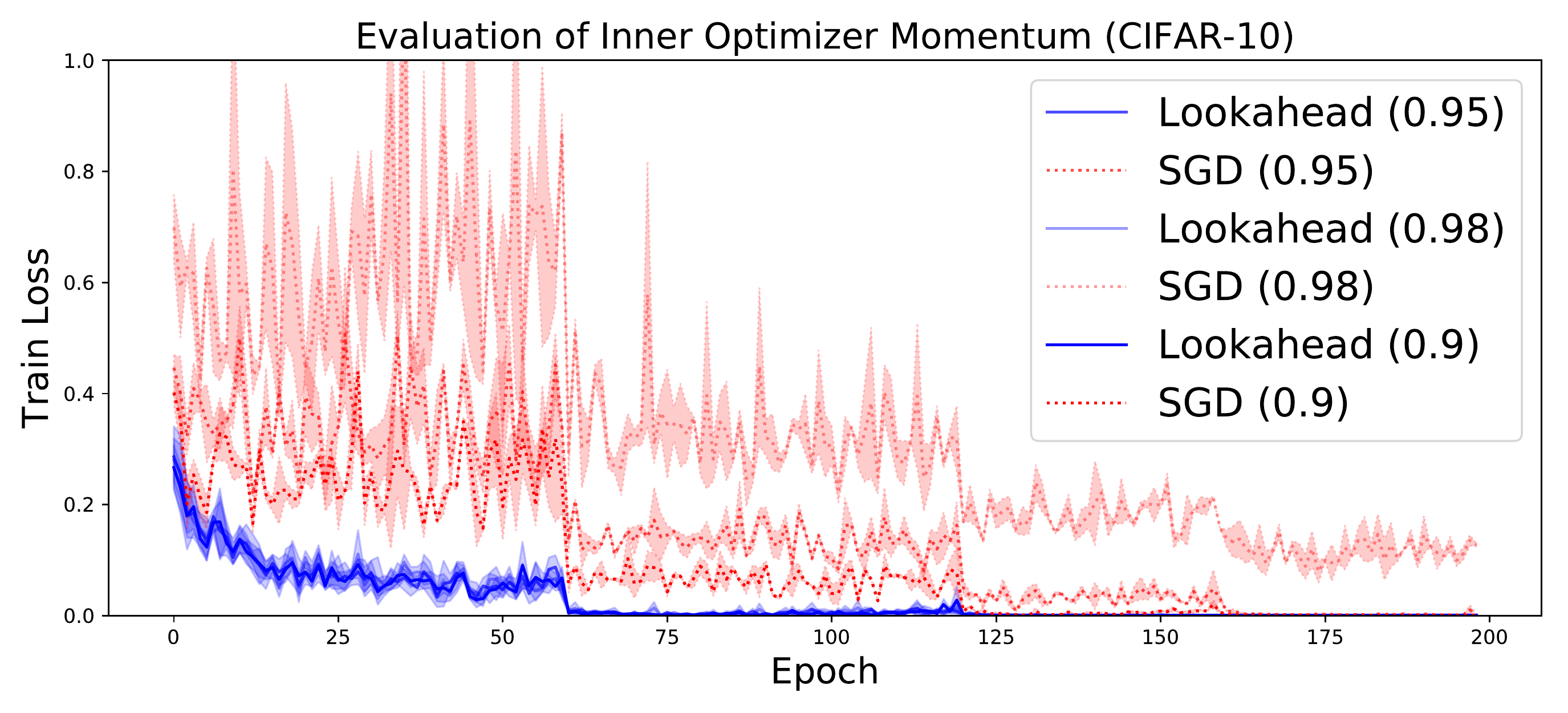}
            \caption[]            {{\small CIFAR-10 Train Loss: Different momentum}}    
                    \end{subfigure}
    \caption{We fix Lookahead parameters and evaluate on different inner optimizers.}
    \label{fig:robust-lr-momentum}
\end{figure*}

\paragraph{Robustness to inner optimization algorithm, $k$, and $\alpha$} 
We demonstrate empirically on the CIFAR dataset that Lookahead consistently delivers fast convergence across different hyperparameter settings.
We fix slow weights step size $\alpha=0.5$ and $k=5$ and run Lookahead on inner SGD optimizers with different learning rates and momentum; results are shown in Figure \ref{fig:robust-lr-momentum}. In general, we observe that Lookahead can train with higher learning rates on the base optimizer with little to no tuning on $k$ and $\alpha$. This agrees with our discussion of variance reduction in Section \ref{sec: noisyquadratic}. 
We also evaluate robustness to the Lookahead hyperparameters by fixing the inner optimizer and evaluating runs with varying updates $k$ and step size $\alpha$; these results are shown in Figure \ref{fig:kalpha_robust}.

\begin{figure}
    \centering
    \begin{minipage}{0.48 \linewidth}
    \includegraphics[width=\linewidth]{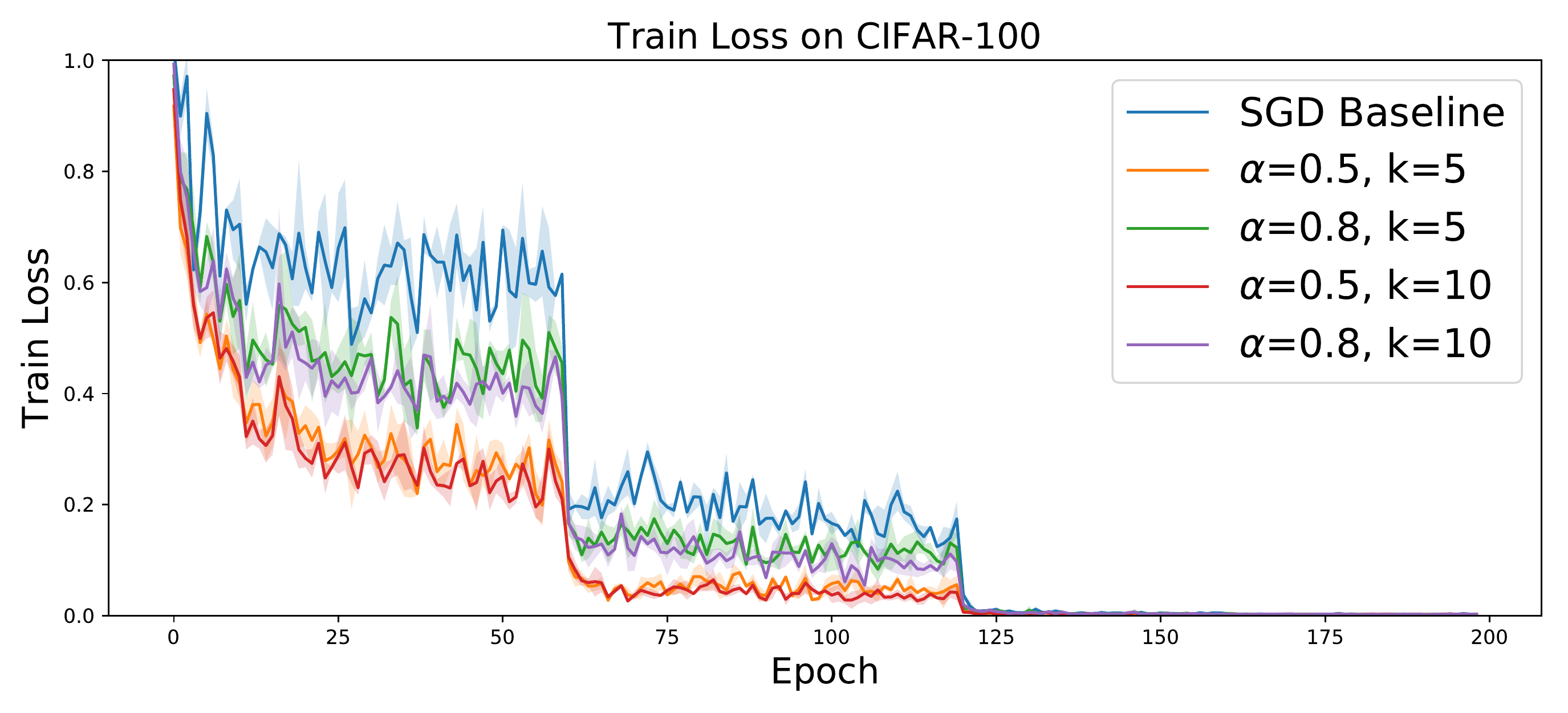}
    \end{minipage} \hfill
\begin{minipage}{0.48 \linewidth}
\begin{table}[H]
\begin{center}
\begin{scriptsize}
\begin{sc}
\begin{tabular}{l c c}
\toprule
\diagbox[width=3em]{k}{$\alpha$}& 0.5 & 0.8 \\
\midrule
5 &  $78.24 \pm .02$ & $78.27 \pm .04$   \\
10  & $78.19 \pm .22$ & $77.94 \pm .22$ \\ 
\bottomrule
\end{tabular}
\end{sc}
\end{scriptsize}
\end{center}
\caption{All settings have higher validation accuracy than SGD (77.72\%)}
\label{tabel:cifar-grid}
\end{table}
\end{minipage}
\caption{CIFAR-100 train loss and final test accuracy with various $k$ and $\alpha$.}
\label{fig:kalpha_robust}
\end{figure}

\paragraph{Inner loop and outer loop evaluation}

\begin{figure}[t!]
    \centering
                            \includegraphics[width=1.0 \textwidth]{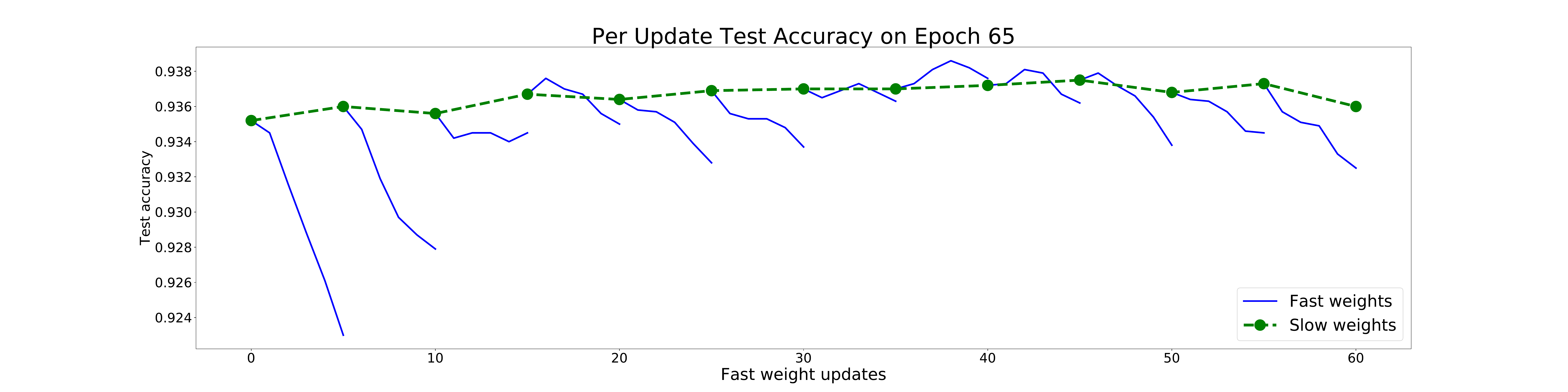}
        \caption{Visualizing Lookahead accuracy for $60$ fast weight updates. We plot the test accuracy after every update (the training accuracy and loss behave similarly). The inner loop update tends to degrade both the training and test accuracy, while the interpolation recovers the original performance.}
        \label{fig:cifar-every-batch}
\end{figure}

To get a better understanding of the Lookahead update, we also plotted the test accuracy for every update on epoch 65 in Figure~\ref{fig:cifar-every-batch}. We found that within each inner loop the fast weights may lead to substantial degradation in task performance---this reflects our analysis of the higher variance of the inner loop update in section \ref{sec: noisyquadratic}. The slow weights step recovers the outer loop variance and restores the test accuracy.

  \section{Conclusion}
\label{sec:conclusion}

In this paper, we present Lookahead, an algorithm that can be combined with any standard optimization method. Our algorithm computes weight updates by \emph{looking ahead} at the sequence of ``fast weights" generated by another optimizer. We illustrate how Lookahead improves convergence by reducing variance and show strong empirical results on many deep learning benchmark datasets and architectures. 

\newpage

 \section*{Acknowledgements}

We'd like to thank Roger Grosse, Guodong Zhang, Denny Wu, Silviu Pitis, David Madras, Jackson Wang, Harris Chan, and Mufan Li for helpful comments on earlier versions of this work. We are also thankful for the many helpful comments from anonymous reviewers. 
\bibliography{references}
\bibliographystyle{plainnat}
\clearpage
\pagebreak

\appendix
\section{Noisy quadratic analysis}
\label{app:noisy}
Here we present the details of the noisy quadratic analysis, and the proof of Proposition~\ref{prop:noisy_quad}.

\paragraph{Stochastic dynamics of SGD} From \citet{wu2018understanding}, we can compute the dynamics of SGD with learning rate $\gamma$ as follows:

\begin{align}
    \E[\bx^{(t+1)}] &= (\bI - \gamma \bA) \E[\bx^{(t)}] \\
    \V[\bx^{(t+1)}] &= (\bI - \gamma \bA)^2\V[\bx^{(t)}] + \gamma^2 \bA^2 \Sigma
\end{align}

\paragraph{Stochastic dynamics of Lookahead SGD} We now compute the dynamics of the slow weights of Lookahead.

\begin{lemma}The Lookahead slow weights have the following trajectories:

\begin{align}
    \E[\bphi_{t+1}] &= [1 - \alpha + \alpha (\bI - \gamma \bA)^k] \E[\bphi_{t}] \\
    \V[\bphi_{t+1}] &= [1-\alpha + \alpha (\bI - \gamma \bA)^k]^2\V[\bphi_t] + \alpha^2 \sum_{i=0}^{k-1} (\bI - \gamma \bA)^{2i}\gamma^2\bA^2\Sigma 
\end{align}
\end{lemma}

\begin{proof}
The expectation trajectory follows from SGD,
\begin{align*}
    \E[\bphi_{t+1}] &= (1-\alpha) \E[\bphi_{t}] + \alpha \E[\btheta_{t,k}] \\
    &= (1-\alpha) \E[\bphi_{t}] + \alpha (\bI - \gamma \bA)^k \E[\bphi_{t}] \\
    &= [1 - \alpha + \alpha (\bI - \gamma \bA)^k] \E[\bphi_{t}]
\end{align*}

For the variance, we can write $\V[\bphi_{t+1}] = (1-\alpha)^2 \V[\bphi_{t}] + \alpha^2 \V[\btheta_{t,k}] + 2\alpha(1-\alpha)\text{cov}(\bphi_t, \btheta_{t,k})$. We proceed by computing the covariance term recursively. For simplicity, we work with a single element, $\theta$, of the vector $\btheta$ (as $\bA$ is diagonal, each element evolves independently).
\begin{align*}
    \text{cov}(\theta_{t,k-1}, \theta_{t,k}) &= \E[(\theta_{t,k-1} - \E[\theta_{t,k-1}])(\theta_{t,k} - \E[\theta_{t,k}])] \\
    &= \E[(\theta_{t,k-1} - \E[\theta_{t,k-1}])(\theta_{t,k} - (1-\gamma a)\E[\theta_{t,k-1}])] \\
    &= \E[\theta_{t,k-1}\theta_{t,k}] - (1 - \gamma a)\E[\theta_{t,k-1}]^2 \\
    &= \E[(1-\gamma a)\theta^2_{t,k-1}] - (1 - \gamma a)\E[\theta_{t,k-1}]^2 \\
    &= (1-\gamma a)\V[\theta_{t,k-1}]
\end{align*}
A similar derivation yields $\text{cov}(\bphi_t, \btheta_{t,k}) = (\bI - \gamma \bA)^k \V[\phi_t]$. After substituting the SGD variance formula and some rearranging we have,
\begin{equation*}
    \V[\bphi_{t+1}] = [1-\alpha + \alpha (\bI - \gamma \bA)^k]^2\V[\bphi_t] + \alpha^2 \sum_{i=0}^{k-1} (\bI - \gamma \bA)^{2i}\gamma^2\bA^2\Sigma
\end{equation*}
\end{proof}
We now proceed with the proof of Proposition~\ref{prop:noisy_quad}.

\begin{proof}
First note that if the learning rate is chosen as specified, then each of the trajectories is a contraction map. By Banach's fixed point theorem, they each have a unique fixed point. Clearly the expectation trajectories contract to zero in each case.

For the variance we can solve for the fixed points directly. For SGD,
\begin{align*}
    V^*_{SGD} &= (1-\gamma \bA)^2 V^*_{SGD} + \gamma \bA^2 \Sigma, \\
    \Rightarrow V^*_{SGD} &= \frac{\gamma^2 \bA^2 \Sigma}{\bI - (\bI - \gamma \bA)^2}.
\end{align*}

For Lookahead, we have,
\begin{align*}
    V^*_{LA} &= [1-\alpha + \alpha (\bI - \gamma \bA)^k]^2 V^*_{LA} +\alpha^2 \sum_{i=0}^{k-1} (\bI - \gamma \bA)^{2i}\gamma^2\bA^2\Sigma \\
    \Rightarrow V^*_{LA} &= \frac{\alpha^2\sum_{i=0}^{k-1} (\bI - \gamma \bA)^{2i}}{\bI - [(1-\alpha)\bI + \alpha (\bI - \gamma \bA)^k]^2} \gamma^2 \bA^2\Sigma \\
\Rightarrow V^*_{LA} &= \frac{\alpha^2 (\bI - (\bI - \gamma\bA)^{2k})}{\bI - [(1 - \alpha) \bI + \alpha(\bI - \gamma\bA)^k]^2} \frac{\gamma^2 \bA^2 \Sigma}{\bI - (\bI - \gamma \bA)^2}
\end{align*}

where for the final equality, we used the identity $\sum_0^k a^i = (1-a^k)/(1-a)$. Some standard manipulations of the denominator on the first term lead to the final solution,
\begin{align*}
    V_{LA}^* = \frac{\alpha^2 (\bI - (\bI - \gamma\bA)^{2k})}{\alpha^2 (\bI - (\bI - \gamma\bA)^{2k}) + 2\alpha(1 - \alpha)(\bI - (\bI - \gamma\bA)^k)} \frac{\gamma^2 \bA^2 \Sigma^2}{\bI - (\bI - \gamma \bA)^2}
\end{align*}
\end{proof}

For the same learning rate, Lookahead will achieve a smaller loss as the variance is reduced more. However, the convergence speed of the expectation term will be slower as we must compare $1-\alpha + \alpha(\bI-\gamma\bA)^k$ to $(\bI-\gamma\bA)^k$ and the latter is always smaller for $\alpha < 1$. In our experiments, we observe that Lookahead typically converges much faster than its inner optimizer. We speculate that the learning rate for the inner optimizer is set sufficiently high such that the variance reduction term is more important--this is the more common regime for neural networks that attain high validation accuracy, as higher initial learning rates are used to overcome the short-horizon bias \citep{wu2018understanding}.

\subsection{Comparing convergence rates} In Figure~\ref{fig:noisy_quad_convergence} we compared the convergence rates of SGD and Lookahead. We specified the eigenvalues of $A$ according to the worst-case model from \citet{li2005sharpness} (also used by \citet{wu2018understanding}) and set $\Sigma = A^{-1}$. We computed the expected loss (Equation~\ref{eqn:noisy_exp_loss}) for learning rates in the range $(0,1)$ for SGD and Lookahead with $\alpha \in (0,1]$, with $k=5$, at time $T=1000$ (by unrolling the above dynamics). We computed the variance fixed point for each learning rate under each optimizer and use this value to compute the optimal loss. Finally, we plot the difference between the expected loss at $T$ and the final loss, as a function of the final loss. This allows us to compare the convergence performance between SGD and Lookahead optimization settings which converge to the same solution.

\paragraph{Further convergence plots} In Figure~\ref{fig:noisy_quad_more_convergence} we present additional plots comparing the convergence performance between SGD and Lookahead. In (a) we show the convergence of Lookahead for a single choice of $\alpha$, where our method is able to outperform SGD even for this fixed value. In (b) we show the convergence after only a few updates. Here SGD outperforms Lookahead for some smaller choices of $\alpha$.  This is because SGD is able to make progress on the expectation more rapidly and reduces this part of the loss quickly --- this is related to the short-horizon bias phenomenon \citep{wu2018understanding}. However, even with only a few updates there are choices of $\alpha$ which are able to outperform SGD.

\begin{figure}
    \begin{minipage}{.5\textwidth}
    \centering
    \includegraphics[width=\linewidth]{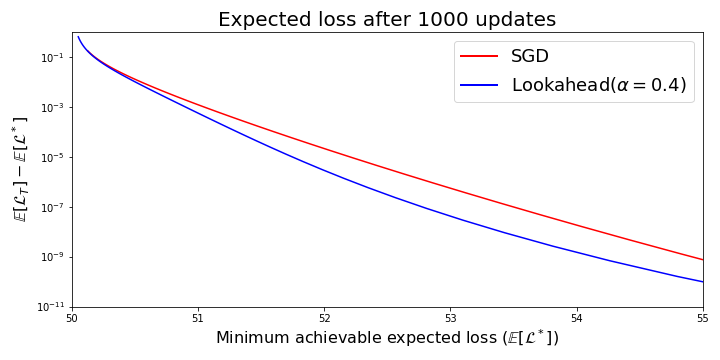}\\(a)
    \end{minipage}    \begin{minipage}{.5\textwidth}
    \centering
    \includegraphics[width=\linewidth]{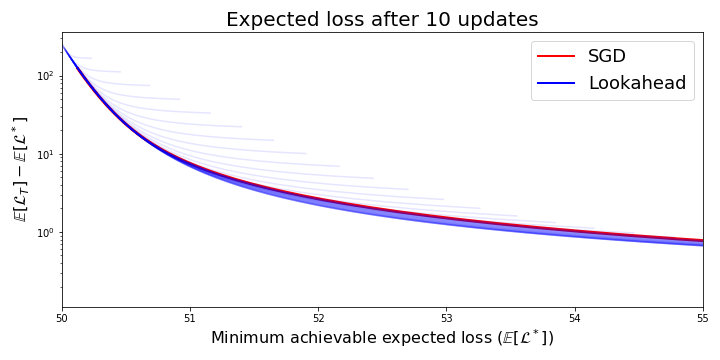}\\(b)
    \end{minipage}
    \caption{Convergence of SGD and Lookahead on the noisy quadratic model. (a): We show the convergence of Lookahead with a single fixed choice of $\alpha=0.4$. (b): We compare the early stage performance of Lookahead to SGD over a range of $\alpha$ values. }
    \label{fig:noisy_quad_more_convergence}
\end{figure}

We also measured the optimal expected loss after some finite time horizon for both Lookahead and SGD in Figure \ref{fig:finite_horizon_nqm}. We performed a fine-grained grid search over the learning rate for SGD and both the learning rate and $\alpha$ for Lookahead (keeping $k=5$ fixed). We evaluated 100 learning rates equally spaced on a log-scale in the range $[10^{-4},10^{-1}]$. For Lookahead, we additionally evaluated 50 $\alpha$ values equally spaced on a log-scale in the range $[10^{-4}, 1]$. For every time horizon, there exists settings of Lookahead that outperform SGD.

\begin{figure}
    \centering
    \includegraphics[width=\linewidth, scale=0.6]{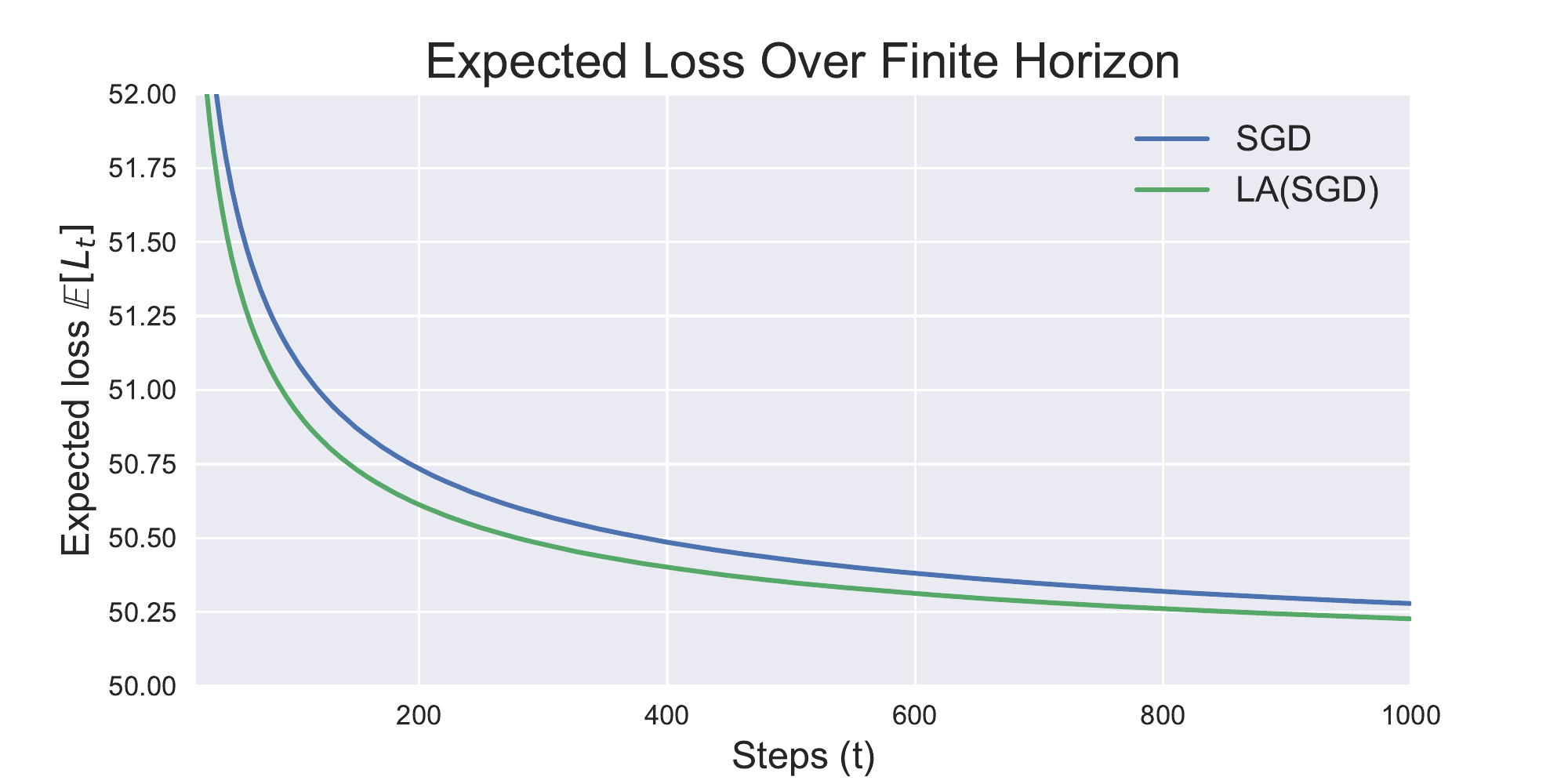}
    \caption{Expected loss of SGD and Lookahead with (constant-through-time) hyperparameters tuned to be optimal at a finite time horizon. At each finite horizon ($x$-axis) we perform a grid search to find the best expected loss of each optimizer. Lookahead dominates SGD over all time horizons.}
    \label{fig:finite_horizon_nqm}
\end{figure}
 \section{Deterministic quadratic convergence analysis}
\label{app:quadratic}

Here we present additional details on the quadratic convergence analysis.

\subsection{Lookahead as a dynamical system}

As in the main text, we will assume that the optimum lies at $\btheta^* = \mathbf{0}$ for simplicity, but the argument easily generalizes. Here we consider the more general case of a quadratic function $f(\bx) = \frac{1}{2} \bx^T A \bx$. We use $\eta$ to denote the CM learning rate and $\beta$ for it's momentum coefficient.

First we can stack together a full set of fast weights and write the following,

\[ 
\left[\begin{array}{c}
\btheta_{t,0} \\
\btheta_{t-1, k} \\
\vdots \\
\btheta_{t-1, 1}
\end{array}\right] = LB^{(k-1)}T \left[\begin{array}{c} \btheta_{t-1, 0} \\
\btheta_{t-2, k} \\
\vdots \\
\btheta_{t-2, 1}
\end{array}\right]
\]

Here, $L$ represents the Lookahead interpolation, $B$ represents the update corresponding to classical momentum in the inner-loop and $T$ is a transition matrix which realigns the fast weight iterates.

Each of these matrices takes the following form,

\[ 
L =
\left[\begin{array}{ccccc}
\alpha I & 0 & \cdots & 0  & (1-\alpha) I \\
I & 0 & \cdots & \cdots & 0 \\
0 & I & \ddots &  \ddots & \vdots \\
\vdots & \ddots & \ddots &  0 & \vdots \\
0 & \cdots & 0 & I & 0
\end{array}\right]
\]

\[ 
B =
\left[\begin{array}{ccccc}
(1 + \beta) I - \eta A & -\beta I & 0 & \cdots  & 0 \\
I & 0 & \cdots & \cdots & 0 \\
0 & I & \ddots &\ddots & \vdots \\
\vdots & \ddots & \ddots &  0 & \vdots \\
0 & \cdots & 0 & I & 0
\end{array}\right]
\]

\[ 
T =
\left[\begin{array}{cccccc}
I - \eta A & \beta I & -\beta I & 0 & \cdots & 0 \\
I & 0 & \cdots & \cdots & 0 & \vdots\\
0 & I & \ddots &  \cdots & \vdots & \vdots\\
\vdots & \ddots & \ddots &  0 & \vdots & \vdots \\
\vdots & \cdots & 0 & I & 0 & 0 \\
0 & \cdots & 0 & 0 & I & 0
\end{array}\right]
\]

Each matrix consists of four blocks. The bottom left block is always an identity matrix that shifts the iterates along one index. The bottom right column is all zeros with the top-right column being non-zero only for $L$ which applies the Lookahead interpolation. The top left row is used to apply the Lookahead/CM updates in each matrix.

After computing the appropriate product of these matrices, we can use standard solvers to compute the eigenvalues which bound the convergence of the linear dynamical system (see e.g. \citet{lessard2016analysis} for an exposition). Finally, note that because this linear dynamical systems corresponds to $k$ updates (or one slow-weight update) we must compute the $k^{th}$ root of the eigenvalues to recover the correct convergence bound.

\subsection{Optimal slow weight step size}
\label{app:optimal-slow}
We present the proof of Proposition 1 for the optimal slow weight step size $\alpha^*$.

\begin{proof}
We compute the derivative with respect to $\alpha$ 
\[\nabla_{\alpha} L(\theta_{t,0} + \alpha (\theta_{t,k} - \theta_{t,0}))
=  (\theta_{t,k} - \theta_{t,0})^T A (\theta_{t,0} + \alpha (\theta_{t,k} - \theta_{t,0})) - (\theta_{t,k} - \theta_{t,0})^T b\]
Setting the derivative to 0 and using $b = A \theta^*$:
\begin{align}
    &\alpha [ (\theta_{t,k} - \theta_{t,0})^T A  (\theta_{t,k} - \theta_{t,0})] =  (\theta_{t,k} - \theta_{t,0})^T A (\theta^* - \theta_{t,0}) \\ 
    \implies &\alpha^* = \arg \min_{\alpha} L(\theta_{t,0} + \alpha (\theta_{t,k} - \theta_{t,0})) = \frac{(\theta_{t,0} - \theta^*)^T A(\theta_{t,0} - \theta_{t, k})}{(\theta_{t,0}  - \theta_{t,k})^T A (\theta_{t,0} - \theta_{t,k})}  
\end{align}
\end{proof}

We approximate the optimal $\theta^* = \theta_{t,k}  - \hat{A}^{-1} \hat{\nabla} L(\theta_{t,k})$, since the Fisher can be viewed as an approximation to the Hessian \cite{martens2014new}. Stochastic gradients are computed on mini-batches used in training so as to not incur additional computational cost.  Because the algorithm with fixed $\alpha$ performs so well, we only did preliminary experiments with an adaptive $\alpha$. We note that the approximation is greedy and incorporating priors on noise and curvature is an interesting direction for future work. \section{Experimental setup}
\label{app:experiments}

Here we present additional details on the experiments appearing in the main paper. 

\subsection{CIFAR-classification}
We run every experiment with three random seeds using the publicly available setup from \citep{devries2017improved}. A reviewer helpfully noted that this implementation of ResNet-18 has wider channels and more parameters than the original. For future work, it would be better to follow the original ResNet architecture for CIFAR classification. However, we observed consistently better convergence with Lookahead across different architectures and have results with other architectures in Appendix \ref{app:swa-comparison}. Our plots show the mean value with error bars of one standard deviation. We use a standard training procedure that is the same as that of \citet{Zagoruyko2016WRN}. That is, images are zero-padded with $4$ pixels on each side and then a random $32 \times 32$ crop is extracted and mirrored horizontally $50\%$ of the time. Inputs are normalized with per-channel means and standard deviations. Lookahead is evaluated on the slow weights of its inner optimizer. To make this evaluation fair, we evaluate the training loss at the end of each epoch by iterating through the training set again, without performing any gradient updates.

We conduct hyperparameter searches over learning rates and weight decay values, making choices based on final validation performance. For SGD, we set the momentum to 0.9 and sweep over the learning rates \{0.01, 0.03, 0.05, 0.1, 0.2, 0.3\} and weight decay values of \{0.0003, 0.001, 0.003\}. The best choice was a learning rate of 0.05 and weight decay of 0.001. We found AdamW \citep{loshchilov2017fixing} to perform better than Adam and refer it to as Adam throughout our CIFAR experiment section. For Adam, we do a grid search on learning rate of \{1e-4, 3e-4, 1e-3, 3e-3, 1e-2\} and weight decay values of \{0.1, 0.3, 1, 3\}. The best choice was a learning rate of 3e-4 and weight decay of 1. For Polyak averaging, we compute the moving average of SGD use the best weight decay from SGD and sweep over the learning rates \{0.05, 0.1, 0.2, 0.3, 0.5\}. The best choice was a learning rate of $0.3$.

For Lookahead, we set the inner optimizer SGD learning rate to $0.1$ and do a grid search over $\alpha = \{0.2, 0.5, 0.8\}$ and $k = \{5, 10\}$, with $\alpha=0.8$ and $k=5$ performing best (though the choice is fairly robust). We report the verison of Lookahead that resets momentum in our CIFAR experiments. 

\subsection{ImageNet}
For the baseline algorithm, we used SGD with a heavy ball momentum of $0.9$. We swept over learning rates in the set: $\{0.01, 0.02, 0.05, 0.1, 0.2, 0.3\}$ and selected a learning rate of 0.1 because it had the highest final validation accuracy.

We directly wrapped Lookahead around the settings provided in the official PyTorch repository repository with $k =5$ and $\alpha=0.5$ (where SGD has a learning rate of 0.1 in the inner loop). Observing the improved convergence of our algorithm, we tested Lookahead with the aggressive learning rate decay schedule (decaying at the 30th, 48th, and 58th epochs). We run our experiments on 4 Nvidia P100 GPUs with a batch size of 256 and weight decay of 1e-4. We used the same settings in the ResNet-152 experiments. 

\subsection{Language modeling}

For the language modeling task we used the model and code provided by \citet{merity2017regularizing}. We used the default settings suggested in this codebase at the time of usage which we report here. The LSTM we trained had 3 layers each containing 1150 hidden units. We used word embeddings of dimension 400.  Within each hidden layer we apply dropout with probability 0.3 and the input embedding layers use dropout with probability 0.65. We applied dropout to the embedding layer itself with probability 0.1. We used the weight drop method proposed in \citet{merity2017regularizing} with probability 0.5. We adopt the regularization proposed in section 4.6 in \citet{merity2017regularizing}: RNN activations have L2 regularization applied to them with a scaling of 2.0, and temporal activation regularization is applied with scaling 1.0. Finally, all weights receive a weight decay of 1.2e-6.

We trained the model using variable sequence lengths and batch sizes of 80. We apply gradient clipping of 0.25 to all optimizers. During training, if validation loss has not decreased for 15 epochs then we reduce the learning rate by half.  Before applying Lookahead, we completed a grid search over the Adam and SGD optimizers to find competitive baseline models. For SGD we did not apply momentum and searched learning rates in the range $\{50, 30, 10, 5, 2.5, 1, 0.1\}$. For Adam we kept the default momentum values of $(\beta_1, \beta_2) = (0.9, 0.999)$ and searched over learning rates in the range $\{0.1, 0.05, 0.01, 0.005, 0.001, 0.0005, 0.0001\}$. We chose the best model by picking the model which achieved the best validation performance at any point during training.

After picking the best SGD/Adam hyperparameters we trained the models again using Lookahead with the best baseline optimizers for the inner-loop. We tried using $k=\{5,10,20\}$ inner-loop updates and $\alpha = \{ 0.2, 0.5, 0.8\}$ interpolation coefficients. Once again, we reported Lookahead's final performance by choosing the parameters which gave the best validation performance during training.

For this task, $\alpha=0.5$ or $\alpha=0.8$ and $k=5$ or $k=10$ worked best. As in our other experiments, we found that Lookahead was largely robust to different choices of $k$ and $\alpha$. We expect that we could achieve even better results with Lookahead if we jointly optimized the hyperparameters of Lookahead and the underlying optimizer.

\subsection{Neural machine translation}
\label{app:nmt}

For this task, we trained on a single TPU core that has 8 workers each with a minibatch size of 2048. We use the default hyperparameters for Adam \citep{vaswani2017attention} and AdaFactor \citep{adafactor} in the experiments. For Lookahead, we did a minor grid search over the learning rate $\{0.02, 0.04, 0.06\}$ and $k=\{5, 10\}$ while setting $\alpha=0.5$. We found learning rate 0.04 and $k=10$ worked best. After we train those models for 250k steps, they can all reach around 27 BLEU on Newstest2014 respectively.

\section{Additional Experiments}

\subsection{Inner Optimizer State}
\label{app:interpolation}

Throughout our paper, we maintain the state of our inner optimizer for simplicity. For SGD with heavy-ball momentum, this corresponds to preserving the momentum. Here, we present a sensitivity study by comparing the convergence of Lookahead when maintaining the momentum, interpolating the momentum, and resetting the momentum. All three improve convergence versus SGD.

\begin{figure}[t]
\centering
\begin{minipage}[t]{0.48 \linewidth}
    \centering
    \includegraphics[width=1. \textwidth]{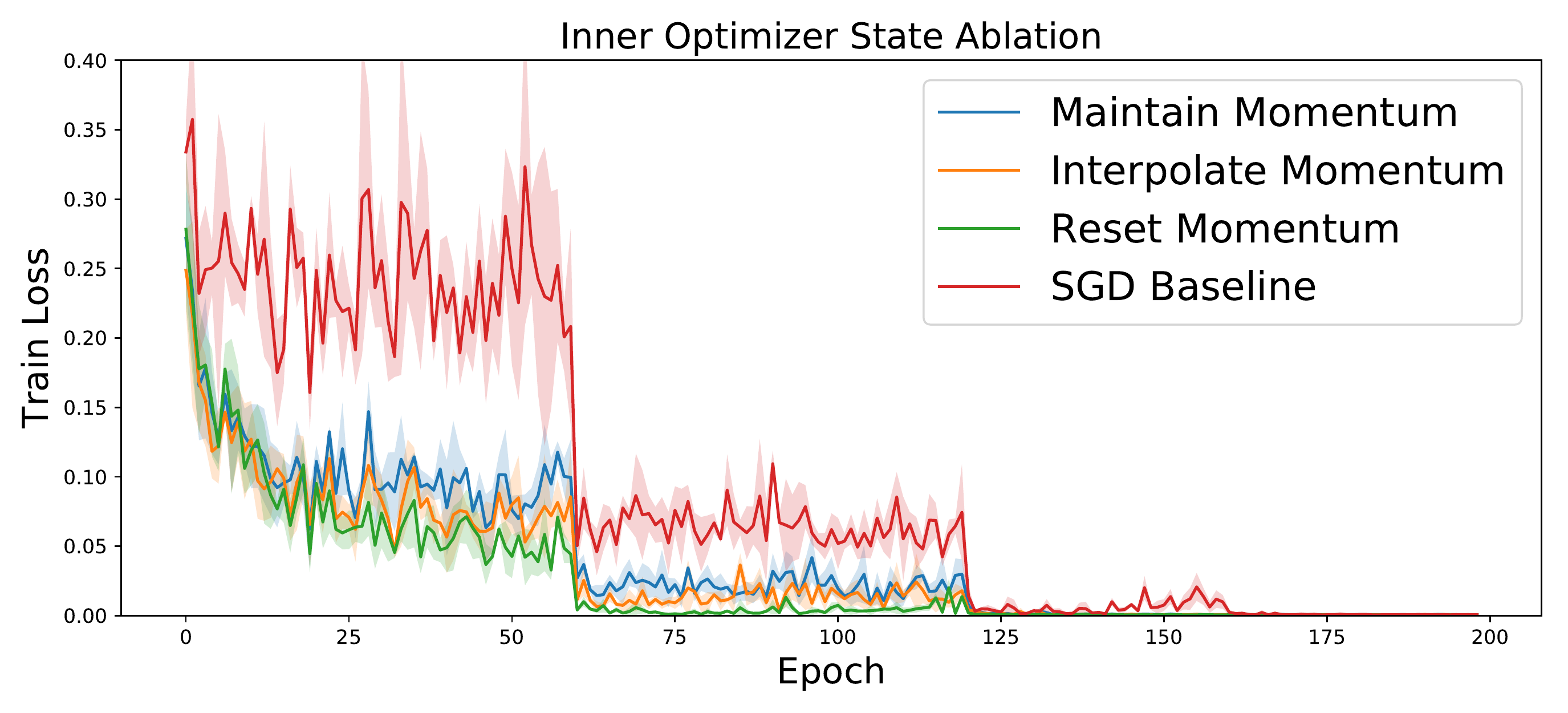}
\end{minipage}
\hfill
\begin{minipage}[t]{0.48 \linewidth}
\begin{table}[H]
\begin{center}
\vspace{-1.1in}
\begin{sc}
\begin{tabular}{l c c  }
\toprule
Optimizer & CIFAR-10  \\
\midrule
Maintain & $95.15 \pm .08$  \\ 
Interpolate & $95.16 \pm .13$ \\
Reset & $94.91 \pm .05$  \\ 
\bottomrule
\end{tabular}
\end{sc}
\end{center}
\caption{CIFAR Final Validation Accuracy.}
\label{tabel:cifar-interpolation}
\end{table}
\end{minipage}
\caption{Evaluation of maintaining, interpolating, and resetting momentum on CIFAR-10}
\end{figure}

\subsection{Validation Accuracy}
We present curves corresponding to the evolution of validation accuracy during training on CIFAR and ImageNet datasets. Though not the focus of our work, we find that faster convergence in training loss does in fact correspond to better validation performance on these datasets.

\begin{figure*}
\begin{minipage}[t]{0.48 \linewidth}
    \centering
    \includegraphics[width=.98\linewidth]{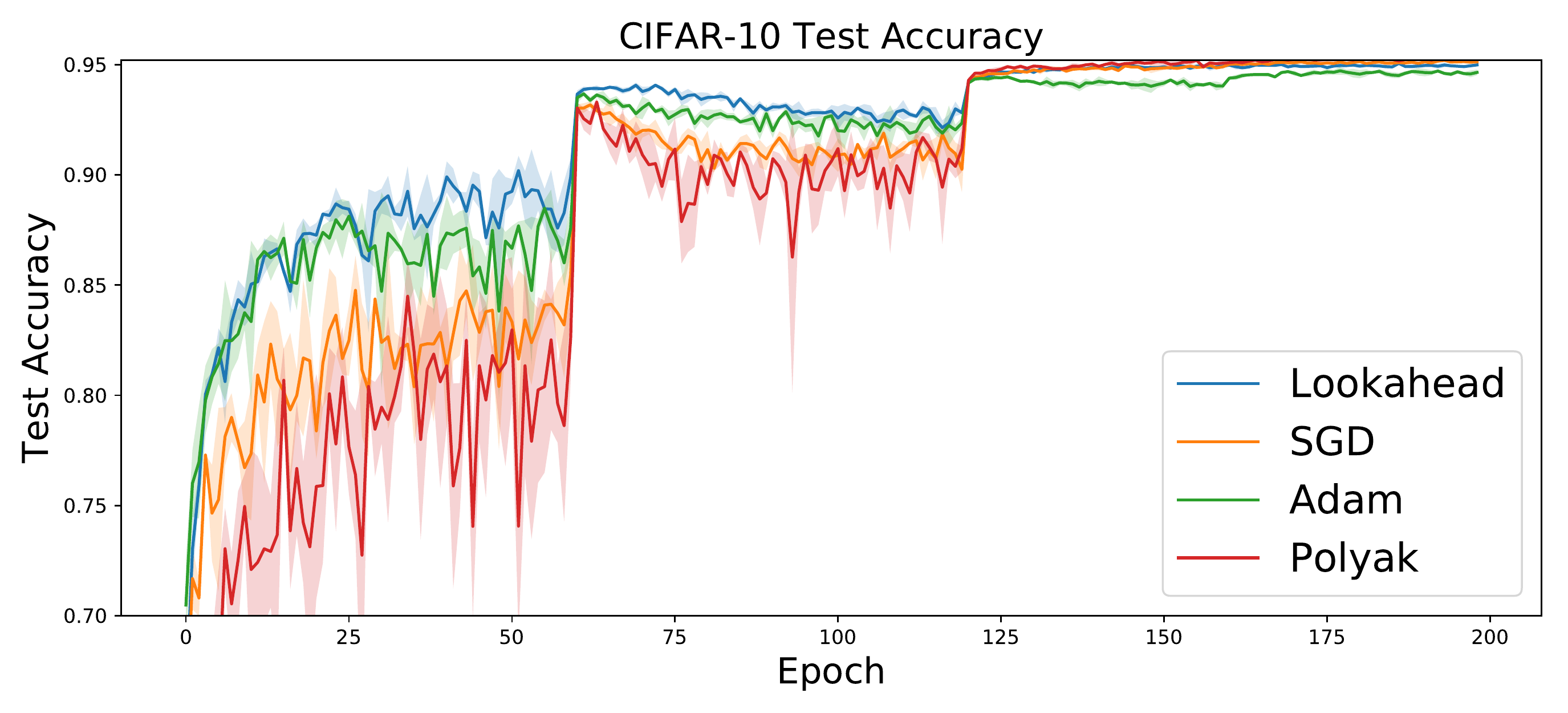}
\end{minipage}
\hfill
\begin{minipage}[t]{0.48 \linewidth}
    \centering
\includegraphics[width=.98\linewidth]{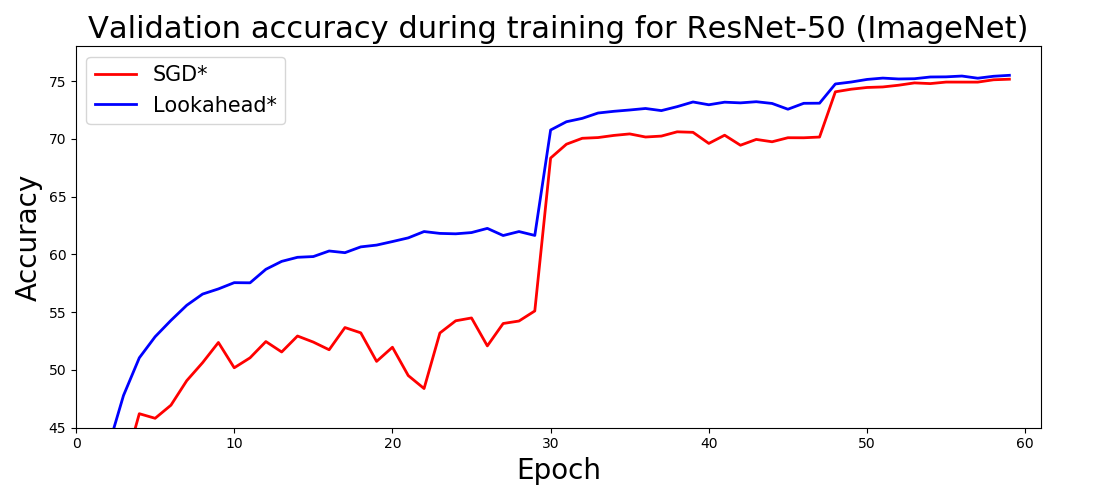}
\end{minipage}
\centering
\caption{Evolution of test accuracy on CIFAR-10 and ImageNet.}
\vskip -0.1in
\label{fig:cifar-val}
\end{figure*}

\subsection{Comparison to Stochastic Weight Averaging}
\label{app:swa-comparison}

In this subsection, we elaborate on differences between Stochastic Weight Averaging (SWA) \cite{izmailov2018averaging} and Lookahead, showing that they serve different purposes but can be complementary. First, SWA and the general family of methods that perform tail averaging \cite{ruppert1988efficient, polyak1992acceleration} requires a choice of when to begin averaging. A choice that is either too early or too late can be detrimental to performance. This is illustrated in Figure \ref{fig:swa_resnet}, where we plot comparisons of the test accuracy of three runs of Lookahead and SWA, using SGD in the inner loop of both algorithms. We use the suggested schedule from the SWA paper, which has higher learning rates than is typical at the end of training. SWA achieves better performance when initialized from epoch 10 compared to epoch 1, due to poor performance from the earlier models. In contrast, Lookahead is used from initialization and does not have this tail averaging start decision. While SWA performs better during the intermediate stages of training (since it is loosely approximating an ensemble by averaging the weights of multiple models), Lookahead with its variance reduction properties achieves better final performance, even with the modified learning rate schedule. Lookahead also computes an exponential moving average of its fast weights rather than the artithmetic average of SWA, which increases the emphasis on recent proposals of weights \cite{martens2014new}.

\begin{wrapfigure}{r}{.6\linewidth}
\vskip -0.2in
\centering
\includegraphics[width=.9 \linewidth]{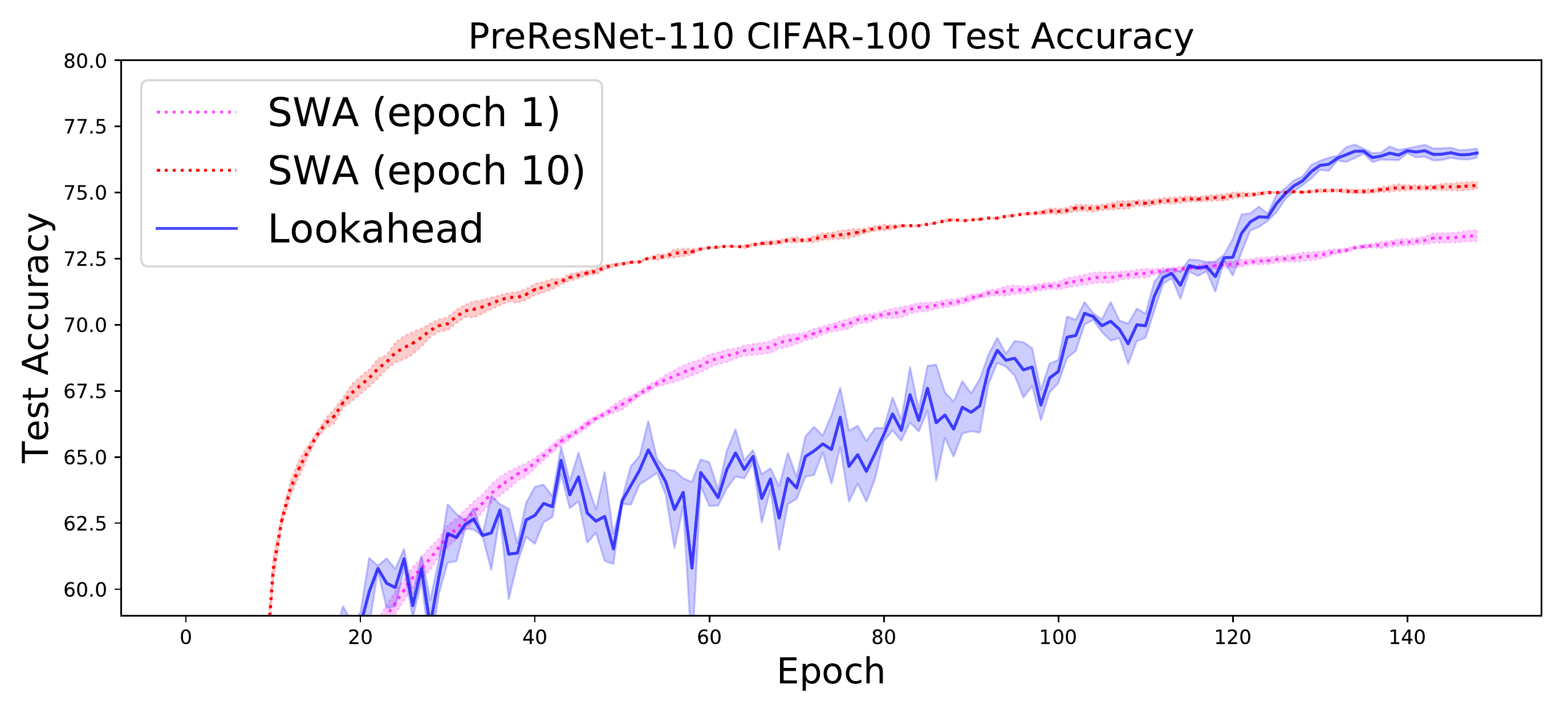}
\caption{Test Accuracy on CIFAR-100 with SWA and Lookahead (PreResNet-110). We follow exactly the hyperparameter settings in their repository and also run Lookahead with $\alpha = 0.8$ and $k=10$. Note that the learning rate schedule uses a learning rate that is higher than is typical at the end of training.}
\label{fig:swa_resnet}
\end{wrapfigure}

We do believe that Lookahead is complementary to SWA and traditional techniques for ensembling models. To this end, we perform three runs with Wide ResNet-28-10 \cite{Zagoruyko2016WRN} on CIFAR-100 with SWA and compare two choices of the inner loop algorithm of SWA. The inner loop algorithms are SGD (as in \cite{izmailov2018averaging}) and Lookahead wrapped around SGD. The runs with Lookahead achieve higher test accuracy throughout training and in the weight averaged network.

\begin{figure*}
\begin{minipage}[t]{0.48 \linewidth}
    \centering
    \includegraphics[width=.98\linewidth]{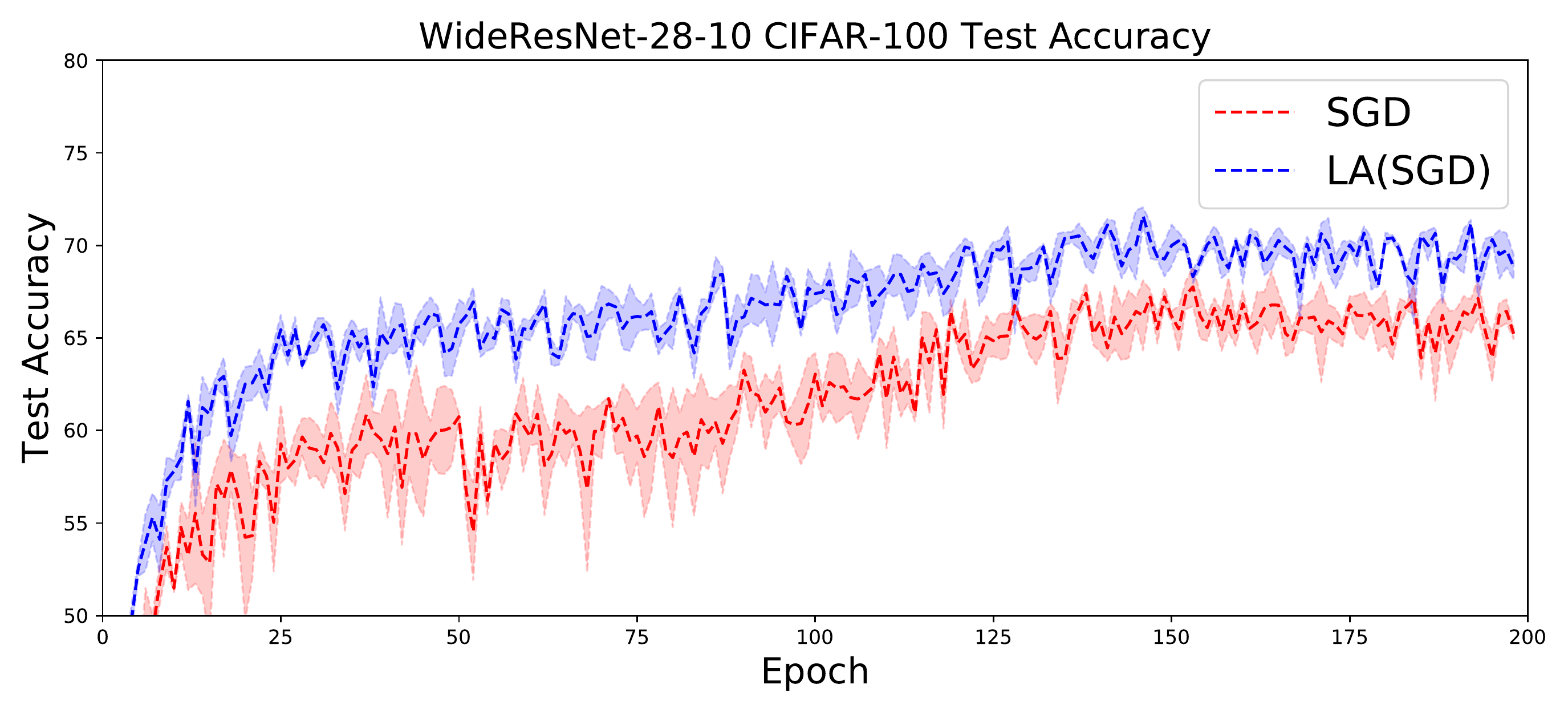}
\end{minipage}
\hfill
\begin{minipage}[t]{0.48 \linewidth}
    \centering
\includegraphics[width=.98\linewidth]{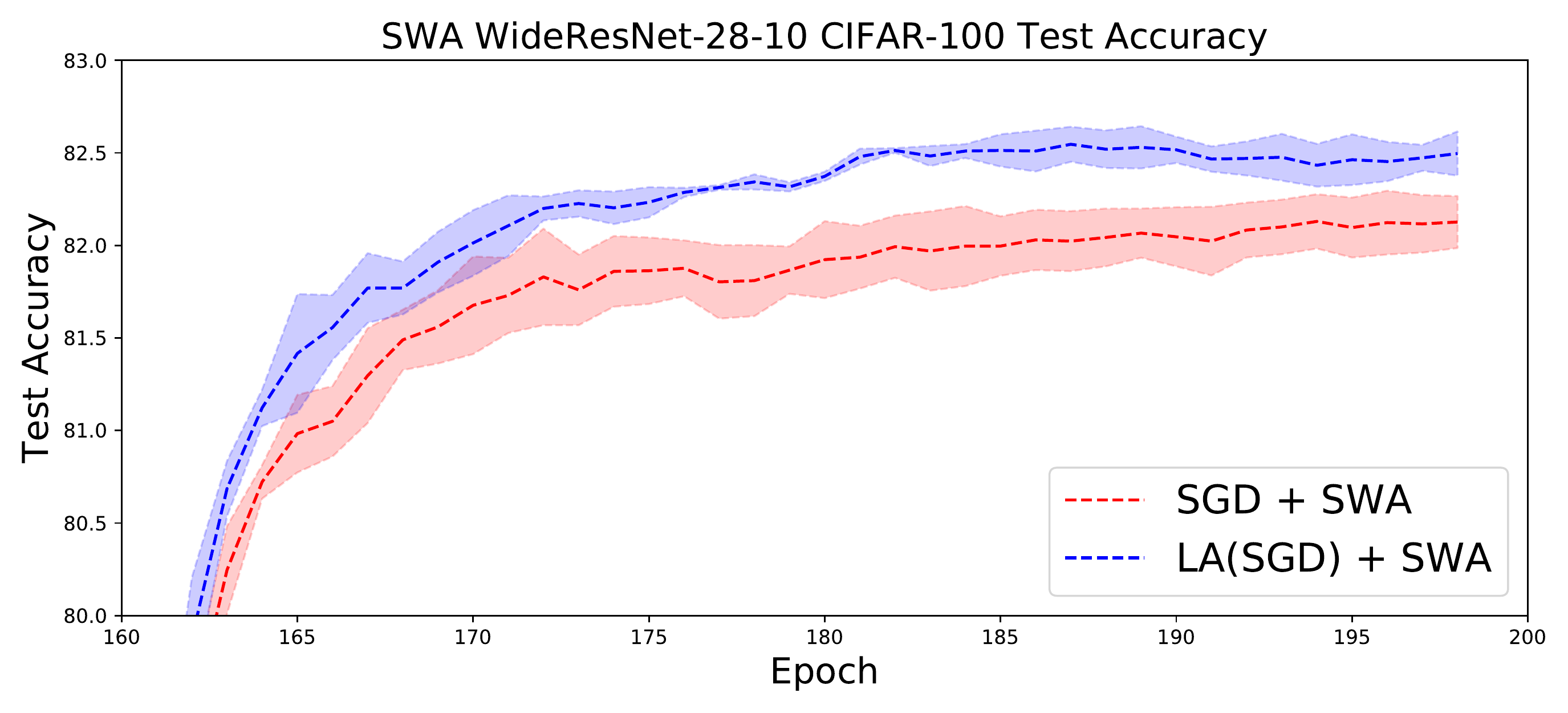}
    \end{minipage}
\centering
\caption{Test Accuracy on CIFAR-100 with SWA and Lookahead (Wide ResNet-28-10). Following Izmailov et al, SWA is started at epoch 161. We plot the accuracy throughout training (left) and the accuracy of the SWA network (right).}
\vskip -0.1in
\label{fig:swa-comparison-wrn}
\end{figure*}


\end{document}